\documentclass[sigconf,screen,nonacm]{acmart}
\usepackage{bbm}

\usepackage{todonotes}
\usepackage{placeins}
\usepackage{framed}
\usepackage{multicol}
\usepackage[ruled]{algorithm2e}
\usepackage{mathtools}

\usepackage{amsmath,amsfonts,bm}

\def\eqref#1{equation~\ref{#1}}

\def\1{\bm{1}}

\DeclareMathAlphabet{\mathsfit}{\encodingdefault}{\sfdefault}{m}{sl}
\SetMathAlphabet{\mathsfit}{bold}{\encodingdefault}{\sfdefault}{bx}{n}

\newcommand{\E}{\mathbb{E}}

\DeclareMathOperator*{\argmax}{arg\,max}
\DeclareMathOperator*{\argmin}{arg\,min}

\newtheorem{theorem}{THEOREM}
\newtheorem{definition}[theorem]{DEFINITION}

\newcommand\Mycomb[2][^n]{\prescript{#1\mkern-0.5mu}{}C_{#2}}
\usepackage{hyperref}
\hypersetup{
    colorlinks=true,
    linkcolor=red,
    citecolor=blue,
    filecolor=blue,      
    urlcolor=blue,
linktocpage}

\usepackage[caption=false]{subfig}

\AtBeginDocument{%
  \providecommand\BibTeX{{%
    \normalfont B\kern-0.5em{\scshape i\kern-0.25em b}\kern-0.8em\TeX}}}

\setcopyright{acmcopyright}
\copyrightyear{2022}
\acmYear{2022}
\acmDOI{ }

\acmConference[WWW '22]{WWW '22: Web Conference}{April 25--29, 2022}{Lyton, France}
\acmBooktitle{Preprint.}
\acmPrice{15.00}
\acmISBN{978-1-4503-XXXX-X/18/06}

\begin{document}

\title{Measuring the Non-Transitivity in Chess}


\author{Ricky Sanjaya}
\affiliation{%
  \institution{University College London}
  \country{rickysanjaya008@gmail.com}}

\author{Jun Wang}
\affiliation{%
  \institution{University College London}
  \country{jun.wang@ucl.ac.uk}}

\author{Yaodong Yang}
\authornote{Corresponding author}
\affiliation{%
  \institution{King's College London}
  \country{yaodong.yang@kcl.ac.uk}
}

\renewcommand{\shortauthors}{Trovato and Tobin, et al.}

\begin{abstract}

It has long been believed that Chess is the \emph{Drosophila} of Artificial Intelligence (AI). Studying Chess can productively provide valid knowledge about complex systems. Although remarkable progress has been made on solving Chess, the geometrical  landscape of Chess in the strategy space is still mysterious. Judging on AI-generated strategies,  researchers hypothesised that the strategy space of Chess possesses a spinning top geometry, with the upright axis representing the \emph{transitive} dimension (e.g., A beats B, B beats C, A beats C), and the radial axis representing the \emph{non-transitive} dimension (e.g., A beats B, B beats C, C beats A). However, it is unclear whether such a hypothesis holds for real-world strategies. In this paper, we quantify the non-transitivity in Chess through real-world data from human players. Specifically, we performed two ways of non-transitivity quantifications---Nash Clustering and counting the number of Rock-Paper-Scissor cycles---on over one billion match data from Lichess and FICS. Our findings positively indicate that the strategy space occupied by real-world Chess strategies demonstrates a spinning top geometry, and more importantly, there exists a strong connection between the degree of non-transitivity and the progression of a Chess player's rating. In particular, high degrees of non-transitivity tend to prevent human players from making progress on their Elo rating, whereas progressions are easier to make at the level of ratings where the degree of non-transitivity is lower. Additionally, we also investigate the implication of  the degree of non-transitivity for  population-based training methods. By considering \emph{fixed-memory Fictitious Play} as a proxy, we reach the conclusion that maintaining large-size and diverse populations of strategies is imperative to training effective AI agents in solving Chess types of games. 
\end{abstract}

%
%



\begin{teaserfigure}
\vspace{-0pt}
\subfloat[]{
  \includegraphics[width= 0.24\textwidth]{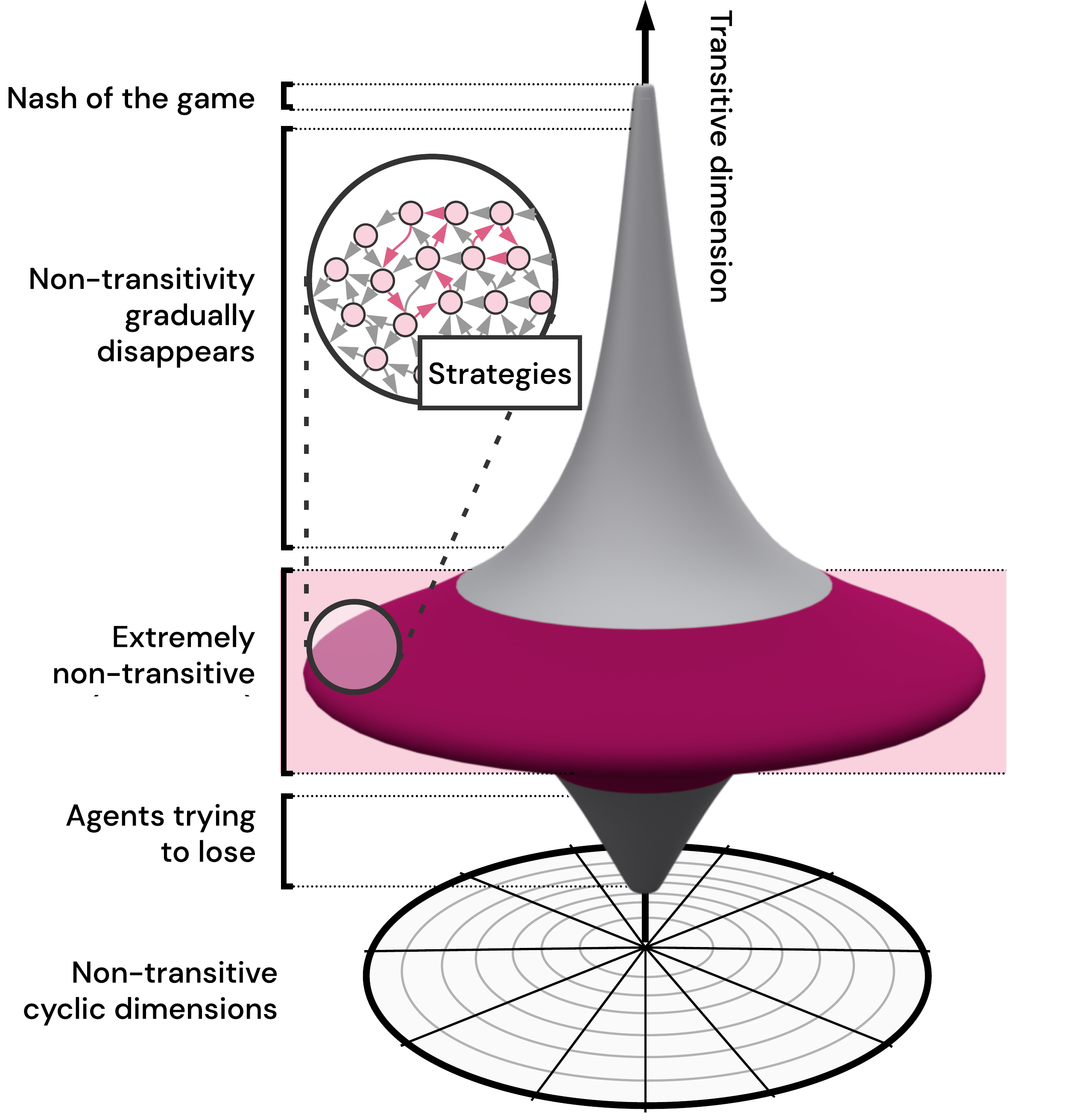}} 
  \subfloat[]{
  \includegraphics[width= 0.47\textwidth]{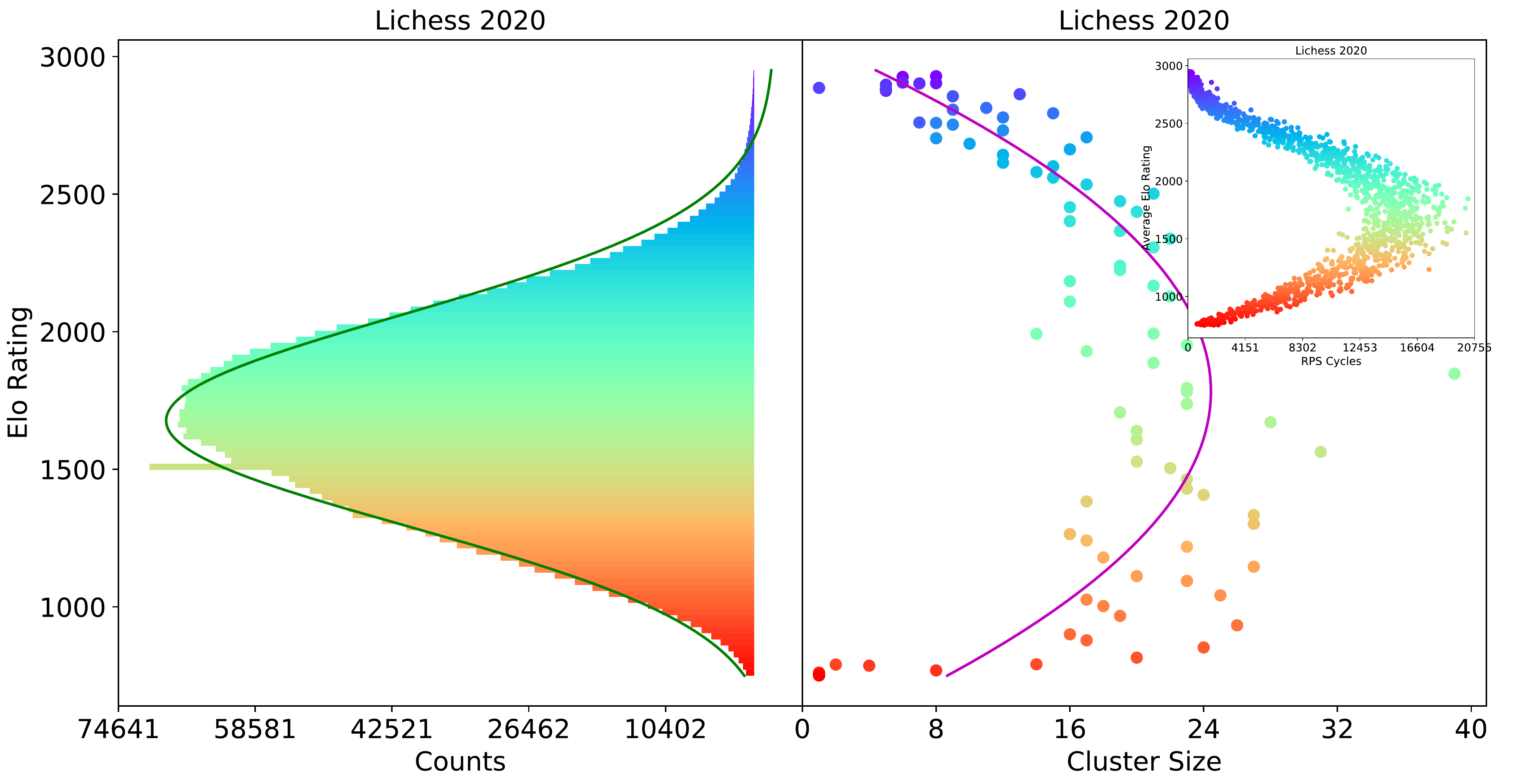}}
  \subfloat[]{
  \includegraphics[width= 0.25\textwidth]{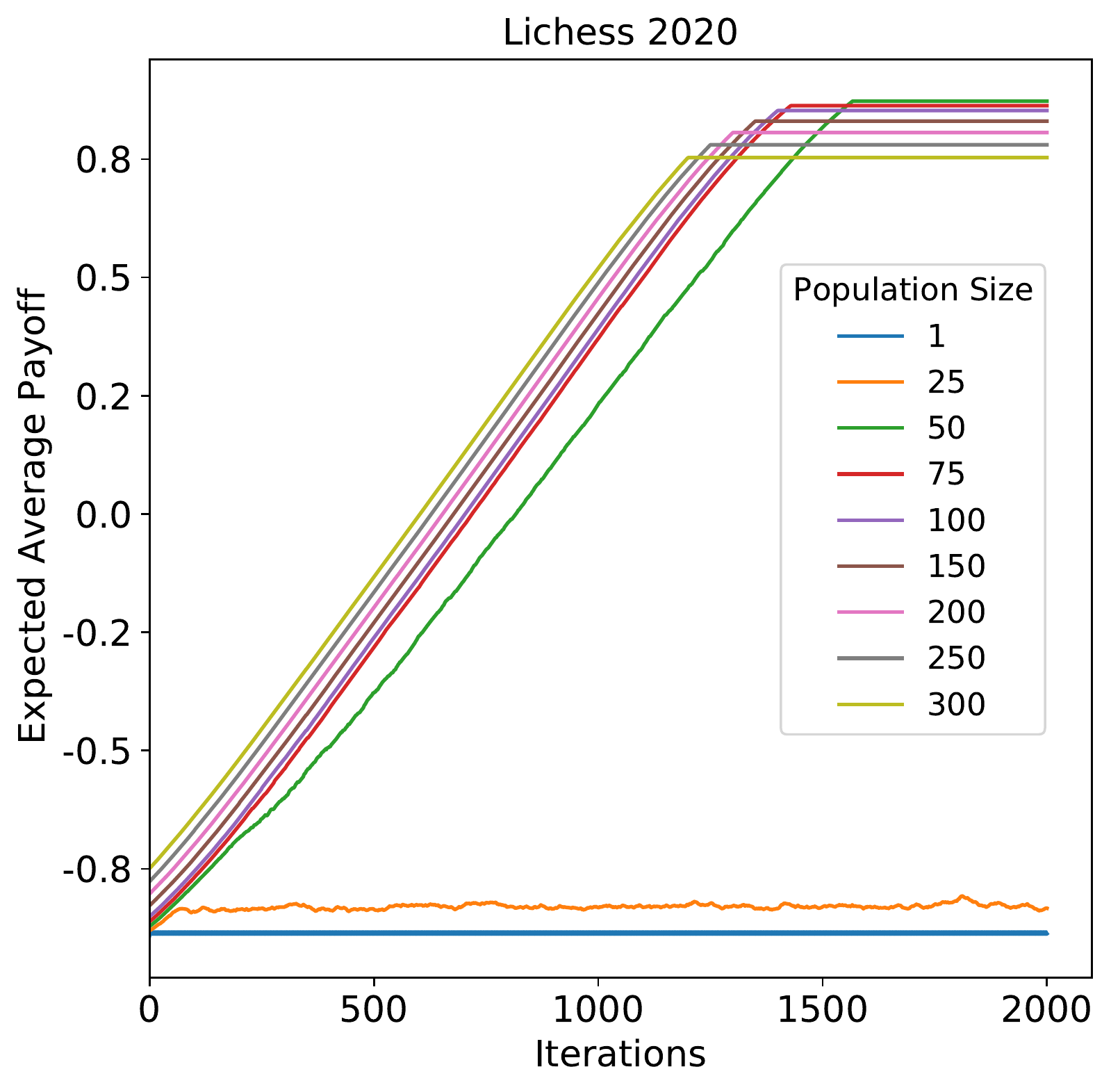}}
  \vspace{-10pt} 
  \caption{We justify the spinning top hypothesis represented by Figure (a) (taken from Figure 1 in \cite{czarnecki2020real}), which states that many real-world games including Chess demonstrate a spinning top  geometry in the strategy space, where the radial axis represents the non-transitivity degree (e.g., A beats B, B beats C, C beats A) and the upright axis represents the transitivity degree (e.g., A beats B, B beats C, A beats C). 
  We study the geometric property of Chess based on one billion game records from human players. Figure (b) shows the comparison between the histogram of Elo ratings (i.e., the transitivity degree) in Lichess 2020  (left) and the degree of non-transitivity at each Elo level (right), where the non-transitivity is measured by both the size of Nash Clusters and the number of Rock-Paper-Scissor cycles (top-right corner). 
  A skewed-normal curve is fitted to illustrate the non-transitivity profile, which verifies the spinning top hypothesis in Figure (a). 
  Specifically, the middle region of the transitive dimension is accompanied by a large degree of non-transitivity, which gradually diminishes as skill evolves towards the high Elo ratings (upward), or diminishes as skill degrades (downward). 
  Notably, the peak of the Elo histogram lies between $1300$ to $1700$, a range where most human players get stuck from making improvements. 
   Furthermore, the peak of the non-transitivity curve coincides with the peak of the Elo histogram; this indicates a strong relationship between the difficulty of playing Chess and  the degree of non-transitivity. Our discovered geometry has important implications for learning. For example, our findings can provide guidance  to improve human players' Elo ratings especially in stages with high degrees of  non-transitivity (e.g., by dedicating more efforts to learn  diverse opening tactics). In Figure (c), we show the performance of population-based training methods with different population sizes.  A phase change in performance  occurs when the population size increases; this justifies  the necessity of maintaining large size and diverse  populations when  training AI agents to solve Chess.} 
  \Description{The results of applying Nash clustering and RPS cycles on Lichess 2020 data, along with using fixed-memory fictitious play to show training behaviour. Finally, a histogram of player Elo ratings in 2020 Lichess games is attached as a comparison.}
  \label{fig:teaser}
  \vspace{15pt}
\end{teaserfigure}

\maketitle

\section{Introduction}

Since the mid 1960s, computer scientists have referred to Chess as the \emph{Drosophila} of AI \cite{ensmenger2012chess}. Similar to the role of fruit fly in genetics studies, Chess provides an accessible, familiar, and relatively simple test bed that nonetheless can be used to produce wide knowledge about other  complex environments. 
As a result, Chess have long been used as benchmarks for the development of AI for decades. The earliest attempt dated back to Shannon's interest in 1950 \cite{shannon1950xxii}.  
In the past decades, remarkable progress has been made on developing AIs that can demonstrate super-human performance on playing Chess; this includes IBM DeepBlue \cite{campbell2002deep} and AlphaZero \cite{silver2017mastering}. 

Despite the algorithmic achievements on Chess, we still have limited knowledge about the geometric landscape of the strategy space of Chess. 
Tied to the strategic geometry are the concepts of \emph{transitivity} and \emph{non-transitivity} \cite{czarnecki2020real}. 
Transitivity refers to how one strategy is  better than other strategies. In transitive games, we know that if strategy A beats strategy B and B beats C, then A can surely beat C. 
On the other hand, 
non-transitive games refers to those where there exists a loop in the preference of strategies. For example, if A can beat B and B beats C, then there is still possibility that  that C can still lose to A in such games.  
One of the simplest purely non-transitive game is Rock-Paper-Scissor game. Yet, 
games in the real world are far more complex; their strategy space often have variations in terms of both transitive and non-transitive components.  
Judging on the strategies that are generated by AIs, researchers have hypothesised that the majority of real-world games possess a spinning top geometry in the strategy space, and in such geometry  (see Figure \ref{fig:teaser}(a)),   the radial axis represents the non-transitivity degree and the upright axis represents the transitivity degree. 
The middle region of the transitive dimension is accompanied by a highly non-transitive dimension, and it gradually diminishes as the skill level either grows towards the Nash Equilibrium \cite{nash1951non},  or degrades to worst-performing  strategies.    

While the idea of non-transitivity is intuitive, how to quantify and measure the size of its dimension remains an open challenge. Early attempt has been made in quantifying non-transitivity \cite{czarnecki2020real}, based on which they proposed the spinning top hypothesis.  However, their sampled strategies are all artificial, in the sense that they only study strategies that are generated by  different algorithmic models of   AIs. While one can argue that both artificial and human strategies are sampled from the same strategy space, they could demonstrate completely different inherent  play styles. For example, \cite{mcilroy2020aligning} showed that Stockfish AI \cite{romstad2017stockfish} did poorly in mimicking the play style of weak human players, despite being depth-limited to mimic that player's strength. Therefore, there is no guarantee that the conclusions from \cite{czarnecki2020real} would apply to human strategies. Consequently, the geometric profile of the strategy space of Chess is still unclear, this directly motivates us to investigate and quantify the degree of non-transitivity on human strategies.

Apart from its popularity in history, 
Chess is also one of the most popular classical games;  it provides the largest amount of well-maintained  human games that are  publicly available in various sources. These databases include Lichess \cite{lichess_about}, which contains games played since 2013, and the Free Internet Chess Server (FICS) which contains games played since 1999 \cite{fics_about}.  By studying over one billion records of human games on Lichess and FICS, in this paper, we measure the non-transitivity of Chess games on real-world data, and investigate its potential implications for training effective AI agents as well as on human skill progression. 
Specifically, we performed two ways of non-transitivity measurements by adopting Nash Clustering and counting the number of Rock-Paper-Scissor cycles.  
Our findings positively indicate that the strategy space occupied by real-world Chess strategies demonstrates a spinning top geometry. More importantly, there exists a strong connection between the degree of non-transitivity and the progression of a Chess player's rating. In particular, high degrees of non-transitivity tend to  prevent human players from making progress on the Elo rating, whereas progressions are easier to make at the level of ratings where the degree of non-transitivity is lower. 
Additionally, we also investigate the implication of the degree of non-transitivity for  population-based training methods. By considering \emph{fixed-memory Fictitious Play} as a proxy, we reach the conclusion that maintaining large-size and diverse populations \cite{yang2021diverse} of strategies is imperative to training effective AI agents in solving Chess type of games, which also matches the empirical findings that are observed based on artificial strategies \cite{czarnecki2020real}.

%
\section{Related Work} \label{section: related works}

Our work is mostly related to \cite{czarnecki2020real}, who proposed the Game of Skill geometry (i.e., the spinning top hypothesis), in which the strategy space of real-world games resembles a spinning top in 2-dimension, where the vertical axis represents the transitive strength, and the horizontal axis represents the degree of non-transitivity. Such hypothesis has  empirically been verified on several real-world games like Hex, Go, Tic-Tac-Toe, and Starcraft. The approach taken by \cite{czarnecki2020real} is to sample strategies uniformly along the transitive strength, and strategies are generated by first running solution algorithms like Minimax Algorithm with Alpha-Beta pruning \cite{russell2002artificial} and Monte-Carlo Tree Search (MCTS) \cite{chaslot2008monte}. Sampling strategies uniformly in terms of transitive strength is then conducted by varying the depths up to which Minimax is performed, and by varying the number of simulations for MCTS. Finally, for Starcraft, sampling strategies are done by using the agents from the training of AlphaStar AI \cite{vinyals2019grandmaster}. However, such sampling of strategies means that the empirical verification is performed  on artificial strategies, and there have been works that suggest an inherent difference in the play styles of AI and humans. For example, \cite{mcilroy2020aligning} showed that Stockfish AI \cite{romstad2017stockfish} did poorly in mimicking the play style of  human players, despite being depth-limited to mimic that player's strength. This motivates us to perform similar investigations, but on human strategies.

This paper focuses the non-transitivity quantification on real-world Chess strategies. However, Chess players are rated primarily using the Elo rating system, which is only valid for the pool of players being considered. As such, there are various systems of player rating computations implemented in the world. The two most used implementations are the ones by the Fédération Internationale des Échecs (FIDE) \cite{fide_charter} and the United States Chess Federation (USCF) \cite{uscf_about}. In its most basic implementation, the Elo rating system consists of two steps, i.e., assigning initial ratings to players and updating the ratings after matches. USCF and FIDE determine a player's starting rating by taking into account multiple factors, including the age of the player and their ratings in other systems \cite{uscf_rating_system, fide_handbook_B022017}. However, the exact rule used in determining the starting rating differs between the two organisations. Furthermore, the update rules used also differs noticeably. These are defined in the handbook of each organisation \cite{uscf_rating_system,fide_handbook_B022017}. On the other hand, online Chess platforms like Lichess and the FICS, which are the source of data for this paper, typically use a refined implementation of the Elo rating system, namely the Glicko \cite{glickman1995glicko} system for FICS, and Glicko-2 \cite{glickman2012example} for Lichess. These systems typically set a fixed starting rating, and then include the computation of a rating deviation for Glicko while updating the player rating after matches. Glicko-2 includes the rating deviation and the rating volatility, both of which represent the reliability of the player's current rating. 

\section{Preliminaries and Notations}

\textbf{Notations. }
For any positive integer $k$, $[k]$ denotes the set $\{1, ..., k\}$. $\Delta_k$ denotes the set of all probability vectors with $k$ elements. For a set $A$ and some function $f$, $\{f(a)\}_{a \in A}$ denotes the set constructed by applying $f$ to each element of $A$. Furthermore, for a set $A= \{A_1, ..., A_{\vert A \vert}\}$, $\{s_i\}_{i \in A}$ denotes $\{s_{A_1}, ..., s_{A_{\vert A \vert}}\}$. Moreover, for an integer $k$, $\{s_i\}_{i = 1}^k$ denotes $\{s_1, ..., s_k\}$. $\mathbb{N}_1$ denotes the set $\{1, 2, ...\}$, i.e., natural numbers starting from $1$. For a set $A = \{A_1, ..., A_{\vert A \vert}\}$, the notation $A_{-j}$ denotes $A$ without the $j^{th}$ element, i.e., $A_{-j} = \{A_i\}_{i \in [\vert A \vert], i \neq j}$. Hence $A = \{A_j, A_{-j}\}$ $\forall j \in [\vert A \vert]$. Furthermore, we use $I[a > b]$ for $a,b \in \mathbb{R}$ to denote the indicator function, where $I[a > b] = 1$ if $a > b$ and $0$ otherwise. Finally, for any integers $m$ and $k$ where $k \leq m$, let $\mathbf{e^m_k}$ denote a vector of length $m$, where the vector is all zeroes, but with a single 1 at the $k^{th}$ element. Likewise, $\mathbf{0}_k$, $\mathbf{1}_k$ denote the vector of zeroes and ones of length $k$ respectively.
\begin{definition}
\label{definition: game theory game}
A game consists of a tuple of $(n, S, M)$ where:
\begin{itemize}
    \item $n \in \mathbb{N}_1$ denotes the number of players.
    \item $S = \prod_{i=1}^n S_i$ denotes the joint strategy space of the game, where $S_i$ is the strategy space of the $i^{th}$ player for $i \in [n]$. A strategy space is a set of strategies a player can adopt.
    \item $M = \{M_i\}_{i = 1}^n$, where $M_i: S \rightarrow \mathbb{R}$ is the payoff function for the $i^{th}$ player. A payoff function maps the outcome of a game resulting from a joint strategy, to a real value, representing the payoff received by that player.
\end{itemize}
\end{definition} 

\textbf{Normal Form (NF) Games. }
In NF games, every player takes a single action simultaneously. Such games are fully defined by a payoff matrix $\boldsymbol{\mathcal{M}}$ where $$\boldsymbol{\mathcal{M}}_{d_1, d_2, ..., d_n} = \big\{M_1(a_1, a_2, ..., a_n), ..., M_n(a_1, a_2, ..., a_n)\big\}$$ 
for $a_k \in A_k$, $d_k \in [\vert A_k \vert]$, $k \in [n]$, $A_k$ denoting the action space of the $k^{th}$ player, and $d_k$ is the index corresponding to $a_k$. $\boldsymbol{\mathcal{M}}$ fully specifies an NF game because it contains all the components of a game by Definition \ref{definition: game theory game}. In an NF game, the strategy adopted by a player refers to the single action taken by that player. Therefore, the action space $A_k$ is equivalent to the strategy space $S_k$ for $k \in [n]$. 

An NF game is zero-sum when $\sum_{z \in \boldsymbol{\mathcal{M}}_{d_1, ..., d_n}} z = 0$ for $d_k \in [\vert A_k \vert]$, $k \in [n]$. For two-player zero-sum NF games, the entries of $\boldsymbol{\mathcal{M}}$ can therefore be represented by a single number, since the payoff of one player is the negative of the other. Additionally, if the game is symmetric, i.e., both players have identical strategy spaces, then the matrix $\boldsymbol{\mathcal{M}}$ is \textbf{skew-symmetric}.

A well-known solution concept that describes the equilibrium of NF games is Nash Equilibrium (NE) \cite{nash1951non,deng2021complexity}. Let $\mathbf{p_i}(s)$ where $s \in S_i$, $\mathbf{p_i} \in \Delta_{\vert S_i \vert}$, $i \in [n]$, denote the probability of player $i$ playing pure strategy $s$ according to probability vector $\mathbf{p_i}$. Also, let $p = \{\mathbf{p_1}, ..., \mathbf{p_n}\}$, and $p[j:\mathbf{k}]$ denoting changing the $j^{th}$ element in $p$ to the probability vector $\mathbf{k}$, where $\mathbf{k} \in \Delta_{\vert S_j \vert}$.

\begin{definition}
\label{NE Constraint}
The set of distributions $p^*$, where $p^* = \{\mathbf{p_1^*}, ..., \mathbf{p_n^*}\}$ and $\mathbf{p_i^*} \in \Delta_{\vert S_i \vert}$ $\forall i \in [n]$, is an NE if $\sum_{s \in S} p^*(s) M_i(s) \geq \sum_{s \in S} p^*[i:\mathbf{e_k^{\vert S_i \vert}}](s) M_i(s)$ $\forall i \in [n]$, $k \in [\vert S_i \vert]$, where $p^*(s) = \prod_{j \in [n]} \mathbf{p_j^*}(s_j)$ and $s = \{s_1, ..., s_n\}$, $s_i \in S_i$ $\forall i \in [n]$.
\end{definition}

Definition \ref{NE Constraint} implies that under NE, no player can increase their expected payoff by unilaterally deviating to a pure strategy. While such NE always exists \cite{nash1951non}, it might not be unique. To guarantee the uniqueness of the obtained NE, we apply the maximum entropy condition when  solving for the NE. The maximum-entropy NE is proved to be unique \cite{balduzzi2018re}.

\begin{theorem}
\label{theorem: symmetric maxent ne}
Any two-player zero-sum symmetric game always has a unique maximum entropy NE given by $\{\mathbf{p^*}, \mathbf{p^*}\}$ where $\mathbf{p^*} \in \Delta_{\vert S_g \vert}$, and $S_g$ is the strategy space for each of the two players \cite{balduzzi2018re}.
\end{theorem}

As a consequence of Theorem \ref{theorem: symmetric maxent ne}, finding the maximum entropy NE of a two-player zero-sum symmetric game amounts to solving the Linear Programming (LP) problem shown in Equation \ref{Special Case Maxent NE LP}.
\begin{align}
    & \mathbf{p^*} = \argmax_\mathbf{p} \sum_{j \in [\vert S_g \vert]} -\mathbf{p}_j \text{ log } \mathbf{p}_j \nonumber \\
   & \text{s. t. }  \ \ \ \ 
     \boldsymbol{\mathcal{M}} \mathbf{p} \leq \mathbf{0}_{\vert S_g \vert} \nonumber \\
    & \ \ \ \ \ \ \ \ \ \ \ \mathbf{p} \geq \mathbf{0}_{\vert S_g \vert} \nonumber \\
    & \ \ \  \ \ \ \ \ \ \ \  \mathbf{1}_{\vert S_g \vert}^\top \mathbf{p} = 1
\label{Special Case Maxent NE LP}
\end{align} 
where $\boldsymbol{\mathcal{M}}$ is the skew-symmetric payoff matrix of the first player

\textbf{Extensive Form (EF) Games. }
EF games \cite{hart1992games} model situations where participants take actions sequentially and interactions happen for more than one step. We consider the case of perfect information, where every player has knowledge of all the other player's actions, and finite horizon, where interactions between participants end in a finite number of steps (denoted as $K$).
In EF games, an action refers to a single choice of move of some player at some stage of the game. Therefore, the term action space, is used to define the available actions to a player at a certain stage of the game.
\begin{definition}
\label{definition: history EF games}
A history at stage $k$, i.e., $h_k$ is defined as $h_k = \{a_1, ..., a_{k-1} \}$, where $a_j = \{a_j^1, ..., a_j^n \}$ for $j \in [k-1]$. $a_j$ is the action profile at stage $j$, denoting actions taken by all players at that stage. 
\end{definition}
The action space of player $i$ at stage $k$ is therefore a function of the history at that stage, i.e., $A_i(h_k)$ for $k \in [K]$, $i \in [n]$, since the available actions for a player depends on how the game goes.
\begin{definition}
\label{definition: strategy EF games}
Let $H_k$ be the set of all possible histories at stage $k$. Let $A_i(H_k) = \bigcup_{h_k \in H_k} A_i(h_k)$ be the collection of actions available to player $i$ at stage $k$ from all possible history at that stage. Let the mapping $s_i^k : H_k \rightarrow A_i(H_k)$ be a mapping from any history at stage $k$ (i.e., $h_k$) to an action available at that stage due to $h_k$, i.e., $s_i^k(h_k) \in A_i(h_k)$  $\forall h_k \in H_k$. A strategy of a player $i$ is defined as $s_i = \bigcup_{k \in [K]} s_i^k$, and $s_i \in S_i$ where $S_i$ is the strategy space of player $i$.
\end{definition}
In EF games, behavioural strategies are defined as a contingency plan for any possible history at any stage of the game. 
According to the Kuhn's theorem \cite{kuhn2009extensive}, each behavioural strategy has a realisation-equivalent mixed strategy in EF games of perfect recall. 
Finally, the payoff function of an EF game is defined as the mapping of every terminal histories, i.e., history at stage $K + 1$, to a sequence of $n$ real numbers, denoting the payoffs of each player, i.e., $\boldsymbol{\mathcal{M}}_i : H_{K+1} \rightarrow R$ for $i \in [n]$, where $H_{K+1}$ is the set of all terminal histories. However, any terminal history is produced by at least one joint strategy adopted by all players. Therefore, the payoff of a player is also a mapping from the joint strategy space to a real number, consistent with Definition \ref{definition: game theory game}.

An EF game can be converted into a NF game. When viewing an EF game as an NF game, the "actions" of a player in the corresponding NF game would be the strategies of that player in the original EF game. Therefore, a $n$-player EF game would have a corresponding $\boldsymbol{\mathcal{M}}$ matrix of $n$ dimensions, where the $i^{th}$ dimension is of length $\vert S_i \vert$. The entries of this matrix correspond to the sequence of $n$ real numbers representing the payoff of each player as a consequence of playing the corresponding joint strategies. 




\textbf{Elo Rating.} The Elo Rating system is a rating system devised by Arpad Elo \cite{elo1978rating}. It is a measure of a participant's performance relative to others within a pool of players. The main idea behind Elo rating is that given 2 players, A and B, the system is to assign a numerical rating to each of them, i.e., $r_A$ and $r_B$ where $r_A, r_B \in \mathbb{R^+}$, such that the winning  probability of $p(A>B)$ is approximated by 
\begin{equation}
\label{Elo formula}
    p(A > B) \approx \frac{1}{1 + \text{exp}(-k(r_A - r_B))} = \sigma \big(k({r_A - r_B}) \big)
\end{equation}
where $k \in \mathbb{R}^+$, $\sigma(x) = {1}/{1 + \text{exp}(-x)}$. Finally, it is worth noting that a game whose payoff function is defined by Equation \ref{Elo formula} is a purely transitive game. 

\section{Non-Transitivity Quantification}

We employ two methods to quantify the degree of non-transitivity; namely, by counting strategic cycles of length three, and by counting the size of Nash Clustering \cite{czarnecki2020real}. 
Before elaborating  the principle of each one of them, we start by introducing the techniques we applied for payoff matrix construction. 

\subsection{Payoff Matrix Construction} \label{subsection: payoff matrix construction}

The raw data for the experiments were obtained from two sources. The primary source is the Lichess \cite{lichess_about} database, which contains games played from 2013 until 2021. 
The second dataset is sourced from the FICS \cite{fics_about} database. For FICS, we use games from 2019 which contains the most number of games. 
Altogether, we collected over one billion records of match data from human players. 
To proceed with such a large amount of data, we introduce several novel procedures to turn match records into payoff matrices.  

The first processing step would be to extract the required attributes from every game data. These would be the outcome of the game (to create the payoff matrix) and the Elo ratings of both players (as the measurement of transitive strength). 

The second step concerns the number of games taken from the whole Lichess database. As of 2021, Lichess database contains a total of more than 2 billion games hosted from 2013. Therefore, a sampling method is required. Games are sampled uniformly across every month and year, and the number of games sampled every month is 120,000 since in January 2013, the database contains only a little more than 120,000 games. To avoid the influence of having a different number of points from each month, the number of games sampled per month is thus kept constant. In earlier years such as 2013 to 2015, sampling uniformly across a month is trivial since games from an entire month can be load into a single data structure. However, for months in later years, games in any given month could reach as many as 10 million. Therefore, a two-stage sampling method is required. The pseudocode for the algorithm can be found in Algorithm \ref{algo: two-stage uniform sampling} of Appendix \ref{appendix: two-staged sampling}. For FICS, the number of games in every month is relatively small and can be sampled directly. 
 
The third processing step concerns discretisation of the whole strategy space to create a skew-symmetric payoff matrix representing the symmetric NF game. For this purpose, a single game outcome should represent the two participants playing two matches, with both participants playing black and white pieces once. We refer to such match-up as a two-way match-up. A single strategy in the naive sense would thus be a combination of one contingency plan (as defined in Definition \ref{definition: strategy EF games}) if the player plays the white pieces with another one if the same player plays the black pieces. However, such naive interpretation of strategy results in a very large strategy space. We thus discretise the strategy space along the transitive dimension. Since we employ real-world data from human players, we use the Elo rating of the corresponding players to measure transitive strength. Therefore, given a pair of Elo ratings $a$ and $b$, the game where the white player is of rating $a$ and black player of rating $b$, and then another game where the black player is of rating $a$ and the white player of rating $b$, corresponds to a single two-way match-up data. Discretisation is conducted by binning Elo ratings, such that given 2 bins, $b_1$ and $b_2$, any two-way match-up where one player's rating falls into $b_1$ and the other player's rating falls into $b_2$, will be treated as a two-way match-up between $b_1$ and $b_2$. The resulting payoff matrix of the NF game is thus outcomes of all two-way match-ups between every possible pair of bins. Since two-way match-ups are considered, the resulting payoff matrix would be skew-symmetric.

However, when searching for the corresponding two-way match-up between 2 bins of rating in the dataset, there could either be multiple games or no games corresponding to this match-up. Consider a match-up between one bin $a = [a_1, a_2]$ and another bin $b = [b_1, b_2]$. To fill in the corresponding entry of the matrix, a representative score for each case is first calculated. Let $rs_{a,b}$ denote the representative score for when the white player rating is in $a$ and the black player rating is in $b$, and $rs_{b,a}$ to denote the other direction. For argument’s sake, consider finding $rs_{a,b}$. Since games are found for bin $a$ against $b$, scores must be expressed from the perspective of $a$. The convention here is $1$, $0$, and $-1$ for a win, draw, and loss respectively. This convention ensures that the entries of the resulting payoff matrix are skew-symmetric. In the case that there are multiple games where the white player's rating falls in $a$ and black player's rating falls in $b$, the representative score is then taken to be the average score from all these games. On the other hand, if there are no such games, the representative score is then the expected score predicted from the Elo rating, i.e., 
\begin{align}
\label{eqn:Expected representative score}
    \E_s[a > b] = 2p(a > b) - 1
\end{align}
where $\E_s[a>b]$ is the expected score of bin $a$ against bin $b$, and $p(a>b)$ is computed using Equation \ref{Elo formula}, by setting $k$ as $\frac{\text{ln}(10)}{400}$, $r_a = \frac{a_1 + a_2}{2}$, and $r_b = \frac{b_1 + b_2}{2}$. Since this score is predicted using Elo rating, it is indifferent towards whether the player is playing black or white pieces. After computing the representative score for both cases, the corresponding entry in the payoff matrix for the two-way match-up between bin $a$ and $b$ is simply the average of the representative score of both cases, i.e., $\frac{rs_{a,b} + rs_{b,a}}{2}$. 
The payoff matrix construction procedure is summarised in Algorithm \ref{algo: Payoff Matrix Construction}.
\FloatBarrier
\begin{algorithm}[t!]
\SetAlgoLined
\textbf{Inputs}: Dataset of games $D$, set of $m$ tuples $B = \{b_1, ..., b_m\}$; \\
Initialise $\boldsymbol{\mathcal{M}}$ as $m \times m$ matrix of zeroes;\\
\For{$i \in [m]$}{
    \For{$j \in [m]$}{
        \uIf{$i \geq j$}{
            \textbf{continue};\\
            }
        \Else{
            $D_{W,B} = D[W \in b_i, B \in b_j]$;\\
            \uIf{$D_{W,B} = \emptyset$}{
                $r_i, r_j = \frac{b_i^L + b_i^H}{2}, \frac{b_j^L + b_j^H}{2}$;\\
                $E_{W,B} = 2p(i > j) - 1$;\\
            }
            \Else{
                $E_{W,B} =$ Av$(D_{W,B})$;\\
            }
            $D_{B,W} = D[B \in b_i, W \in b_j]$;\\
            \uIf{$D_{B,W} = \emptyset$}{
                $r_i, r_j = \frac{b_i^L + b_i^H}{2}, \frac{b_j^L + b_j^H}{2}$;\\
                $E_{B,W} = 2p(i > j) - 1$;\\
            }
            \Else{
                $E_{B,W} =$ Av$(D_{B,W})$;\\
            }
            $\boldsymbol{\mathcal{M}}_{i,j}, \boldsymbol{\mathcal{M}}_{j,i} = \frac{E_{W,B} + E_{B,W}}{2}, -\frac{E_{W,B} + E_{B,W}}{2}$;\\
        }
    }
}
Output $\boldsymbol{\mathcal{M}}$ as the payoff matrix;\\
\caption{Chess Payoff Matrix Construction}
\label{algo: Payoff Matrix Construction}
\end{algorithm}
\FloatBarrier
In Algorithm \ref{algo: Payoff Matrix Construction}, $D[W \in b_i, B \in b_j]$ collects all games in $D$ where the white player's rating $W$ is $b_i^L \leq W \leq b_i^H$ and the black player's rating $B$ is $b_j^L \leq B \leq b_j^H$. Av$(D_{W,B})$ averages game scores from the perspective of the white player, i.e., the score is $1$, $0$, $-1$ if the white player wins, draws, or losses respectively. Finally, in computing $p(i > j)$ following Equation \ref{Elo formula}, $r_i$ and $r_j$ used in the equation follows from the $r_i$ and $r_j$ computed in the previous line in the algorithm.

\subsection{Nash Clustering}

\begin{figure*}[!t]
\vspace{-5pt}
    \centering
    \subfloat[]{\includegraphics[width = 0.33\textwidth]{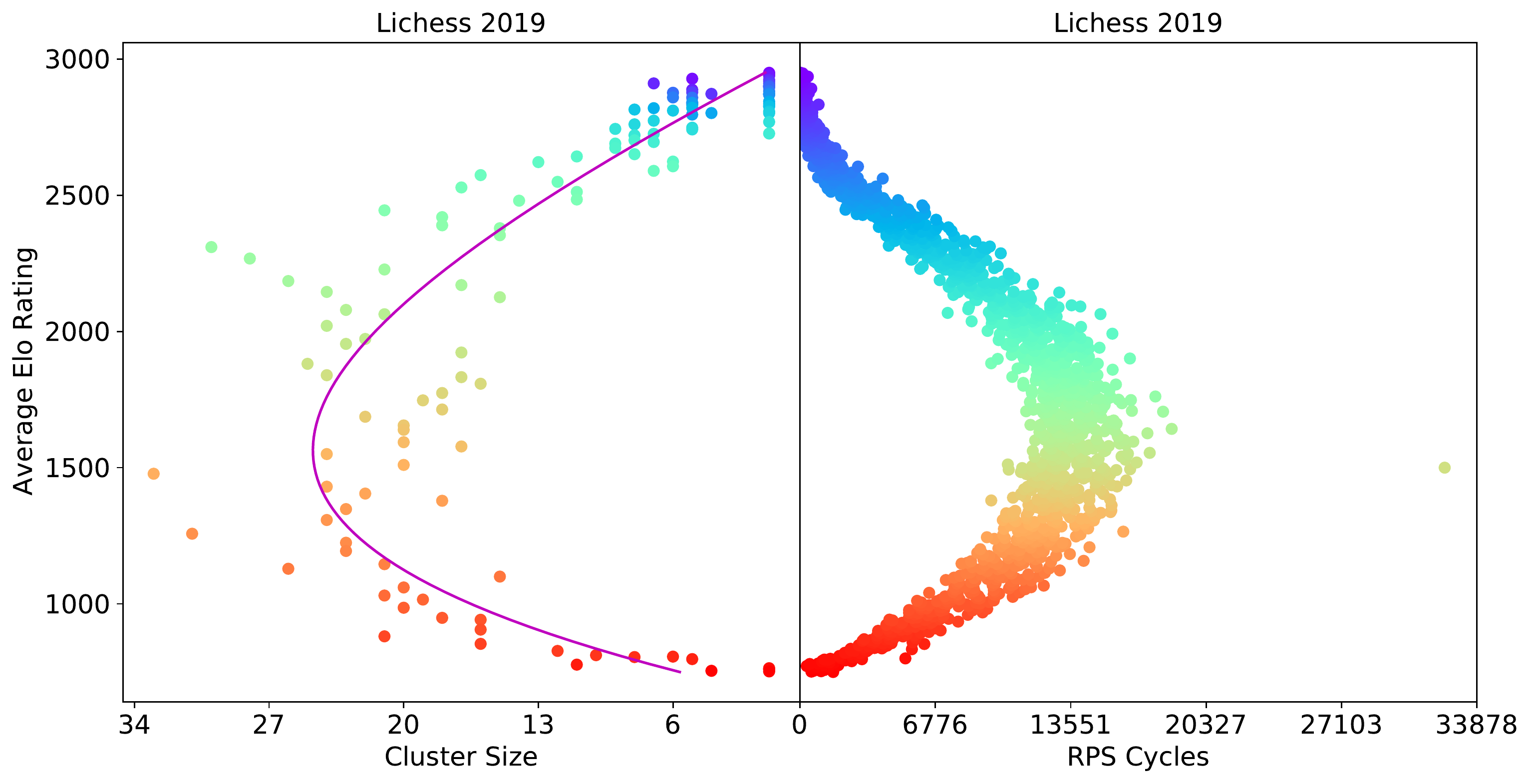}} %
    \subfloat[]{\includegraphics[width = 0.33\textwidth]{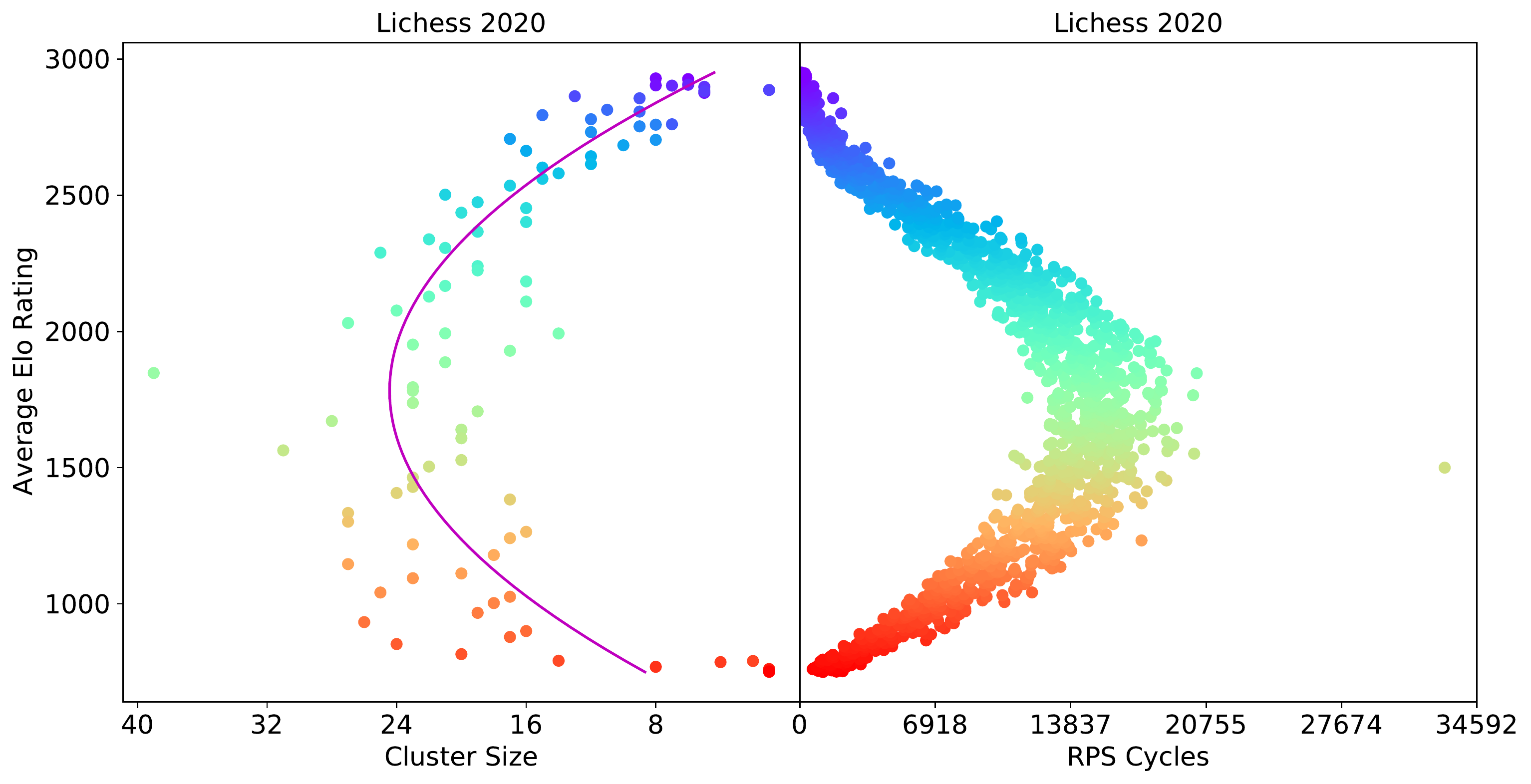}} %
    \subfloat[]{\includegraphics[width = 0.33\textwidth]{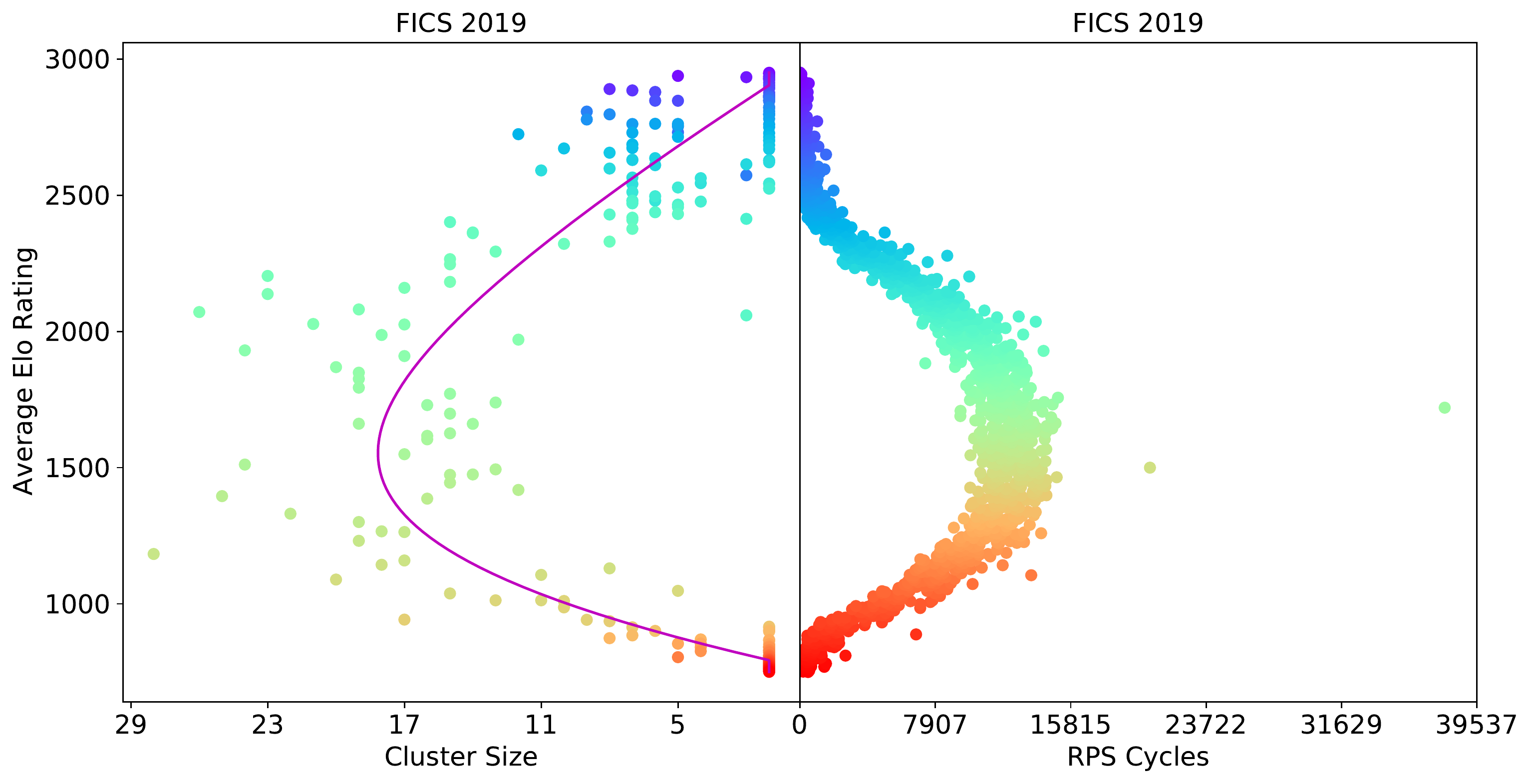}}
    \vspace{-10pt}
    \caption{Nash Clustering (left of each plot) and RPS Cycles (right of each plot) on Lichess 2019 (a), Lichess 2020 (b) and FICS 2019 (c). See Figure \ref{fig: full results} of Appendix \ref{appendix: full results} for the full version of the results, which include Lichess 2013 - 2020 and FICS 2019. The spinning top geometry is observed in both Nash Clustering and RPS Cycles plots, for all the years. Furthermore, both Nash cluster size and number of RPS cycles peaks at a similar Elo rating range of 1300 to 1700 for all the years. This provides comprehensive evidence that the strategies adopted in real-world games do not vary throughout the years or from one online platform to another, and that the non-transitivity of real-world strategies is highest in the range of 1300 to 1700 Elo rating.}
    \label{fig: nash clustering and rps cycles elo}
    \Description{Nash Clustering and counting RPS Cycles on Lichess data from 2019 and 2020, and FICS data from 2019; empirical evidence to illustrate the existence of the spinning top geometry in real-world Chess strategy space.}
    \vspace{-0pt}
\end{figure*}


\begin{figure}[!t]
    \centering
    \includegraphics[width = \linewidth]{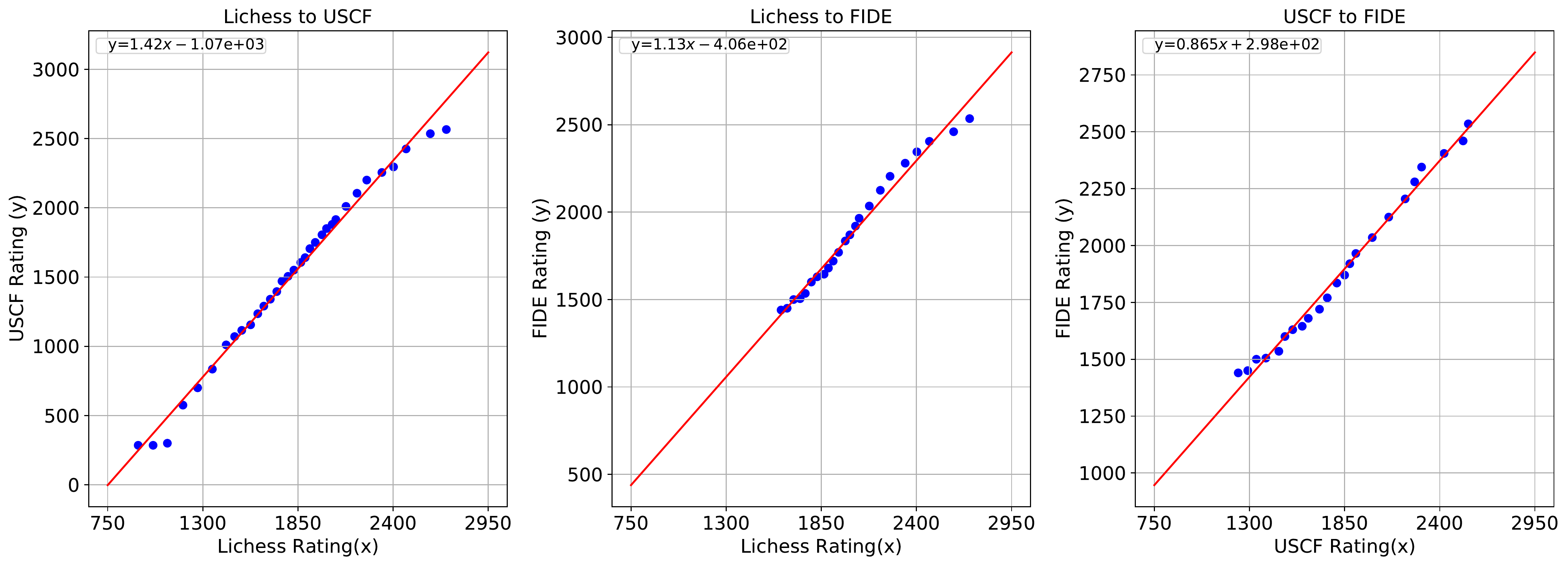}
    \vspace{-20pt}
    \caption{Chess Rating Mappings between Lichess, USCF, and FIDE. The relation between Lichess, USCF, and FIDE ratings are largely linear and monotonically increasing, and the linear model fits well to the data.}
    \Description{Chess player rating mapping between Lichess system, FIDE system, and USCF system.}
    \label{fig: lr regression}
    \vspace{-0pt}
\end{figure}

Following the payoff matrix constructed from the previous step, the first method of quantifying non-transitivity is through Nash Clustering \cite{czarnecki2020real}. Nash Clustering is derived from the layered game geometry, which states that the strategy space of a game can be clustered into an ordered list of layers such that any strategy in preceding layers is strictly better than any strategy in the succeeding layers. A game in which the strategy space can be clustered in this way is known as a layered finite game of skill.

\begin{definition}
\label{definition: k-layered}
A two-player zero-sum symmetric game is defined as a $k$-layered finite game of skill if the strategy space of each player $S_g$ can be factored into $k$ layers, i.e., $L_i$ $\forall i \in [k]$ where $\bigcup_{i \in [k]} L_i = S_g$, $L_i \cap L_j = \emptyset \ \forall i \neq j$, $i,j \in [k]$, and the layers are such that $M_1(s_1, s_2) > 0$ $\forall s_1 \in L_i$, $s_2 \in L_j$, $\forall i,j \in [k]$, $i < j$ and $\exists z \in \mathbb{N}_1$ such that $\forall i < z$, $\vert L_i \vert \leq \vert L_{i+1} \vert$, $\forall i \geq z$, $\vert L_i \vert \geq \vert L_{i+1} \vert$.
\end{definition}

A consequence of Definition \ref{definition: k-layered} is that strategies in earlier layers would have higher win-rate than those in succeeding layers. For any strategy $s \in S_g$, the corresponding win-rate $TS(s)$ is
\begin{equation}
\label{strategy win-rate equation}
    TS(s) = \frac{1}{\vert S_g \vert - 1}\sum_{z \in S_g} I[M_1(s,z) > 0]. 
\end{equation} 
Therefore, if one define the transitive strength of a layer to be the average win-rate of the strategies within that layer, the transitive strength of the layers would be ordered in descending order. Hence, intuitively, one can see the $k$-layered finite game of skill geometry as separating the transitivity and non-transitivity components of the strategy space. The non-transitivity is contained within each layer, whereas the transitivity variation takes place across layers. It is thus intuitive to use the layer sizes as a measure of non-transitivity.

Furthermore, from Definition \ref{definition: k-layered}, the layer sizes are such that it increases monotonically up to a certain layer, from which it then decreases monotonically. Thus, when the transitive strength of each layer is plotted against the size, a 2-dimensional spinning top structure would be observed if the change in sizes are strictly monotonic and $k > 2$. It is also because of this reason, \cite{czarnecki2020real} named this geometry the spinning top geometry.

For every finite game, while there always exists $k \geq 1$ for which the game is a $k$-layered finite game of skill, a value of $k > 1$ does not always exist for any game. Furthermore, the layered geometry is not useful if $k = 1$. Therefore, an alternative relaxation to the $k$-layered geometry is proposed where a layer is required to be "overall better" than the succeeding layers. However, it is not necessary that every strategy in a given layer can beat every strategy in any succeeding layers. Such layer structure is obtained by Nash Clustering.

Let $\text{Nash}(\boldsymbol{\mathcal{M}} \vert X)$ where $X \subset S_g$ denote the symmetric NE of a two-player zero-sum symmetric game, when both players are restricted to only the strategies in $X$, and $\boldsymbol{\mathcal{M}}$ be the payoff matrix from the perspective of player 1. Furthermore, let $\text{supp}(\text{Nash}(\boldsymbol{\mathcal{M}} \vert X))$ denote the set of pure strategies that are in support of the symmetric NE. Finally, let the set notation $A \backslash B$ denote the set $A$ but excluding the elements that are in set $B$.

\begin{definition}
\label{Nash Clustering Definition}
Nash Clustering of a finite two-player zero-sum symmetric game produces a set of clusters $C$ defined as 
\begin{align}
C &= \{C_j : j \in \mathbb{N}_1, C_j \neq \emptyset\},  \ \  \text{where}  \nonumber\\ C_0 &= \emptyset, \ \  C_i = \text{supp}\Big(\text{Nash}\big(\boldsymbol{\mathcal{M}} \big| S_g \backslash \bigcup_{k = 0}^{i-1} C_k \big)\Big), \ \ \  \forall i \geq 1.
\end{align}
\end{definition} 
Furthermore, to ensure the Nash Clustering procedure is unique, this paper uses maximum entropy NE  in Equation \ref{Special Case Maxent NE LP}. Following Nash Clustering, the measure of non-transitivity would thus be the size or number of strategies in each cluster.  

One can  apply the Relative Population Performance (RPP) \cite{balduzzi2019open} to determine if a layer or cluster is overall better than another cluster. Given 2 Nash clusters $C_i$ and $C_j$, $C_i$ is overall better or wins against $C_j$ if $\text{RPP}(C_i, C_j) > 0$. When using RPP as a cluster relative performance measure, the requirement that any given cluster is overall better than any succeeding clusters proves to hold \cite{czarnecki2020real}. Specifically, one can use the fraction of clusters beaten, or win-rate, defined in Equation \ref{RPP win-rate equation}, as the measure of transitive strength of each cluster.
\begin{equation}
\label{RPP win-rate equation}
TS(C_a) = \frac{1}{\vert C \vert - 1} \sum_{C_i \in C} I \Big[\text{RPP}\big(C_a, C_i\big) > 0\Big]
\end{equation}
where  $I$ is the indicator function. We refer to the above win-rate as the RPP win-rate. Since any layer  beats all succeeding layers, the win-rates of the Nash clusters would occur in descending order. 

Alternatively, we can apply the average Elo rating in Equation \ref{Elo formula} of the pure strategies inside each Nash cluster to measure transitive strength of the cluster. 
Specifically, given the payoff matrix from  Section \ref{subsection: payoff matrix construction},  each pure strategy corresponds to an Elo rating bin, and the transitive strength of a Nash cluster is thus computed as the average Elo rating of the bins inside that cluster.  
Let $C_k \in C$ be a Nash cluster, and $c_k$ be the integer indexes of the pure strategies in support of $C_k$, $B = \{b_1, ..., b_m\}$ be the bins used when constructing the payoff matrix, where each bin $b_i$ is has a lower and upper bound of $(b_i^L, b_i^H)$, and the bins corresponding to $C_k$ is $\{b_j\}_{j \in c_k}$. The average Elo rating of Nash cluster $C_k$ is defined as 
\begin{equation}
\label{eqn: Nash cluster average label}
    TS_{\text{Elo}}(C_k) = \frac{1}{\vert C_k \vert} \sum_{j \in c_k} \frac{b_j^L + b_j^H}{2}. 
\end{equation}

Given the average Elo rating against the cluster size for every Nash cluster from Nash Clustering, the final step would be to fit a curve to represent the spinning top geometry. We use the method of fitting an affine-transformed skewed normal distribution to minimise Mean Squared Error (MSE). The results of applying Nash Clustering is shown in Figure \ref{fig: nash clustering and rps cycles elo}.

Furthermore, we also investigate using win-rates as a measure of transitive strength, defined in Equation \ref{RPP win-rate equation}. However, instead of using RPP to define relative performance between clusters, we use an alternative, more efficient measure, that achieves the same result when applied on Nash clusters. This measure is named the Nash Population Performance (NPP).
\begin{definition}
\label{NPP definition}
Let $C = \{C_1, ..., C_{\vert C \vert}\}$ be the Nash Clustering of a two-player zero-sum symmetric game with a payoff matrix of $\boldsymbol{\mathcal{M}}$. $\text{NPP}(C_i, C_j) = \mathbf{p_i}^\top \boldsymbol{\mathcal{M}} \mathbf{p_j}$ where $\mathbf{p_k}= \text{Nash}(\boldsymbol{\mathcal{M}} \vert \bigcup_{r = k}^{\vert C \vert} C_r$ for $k \in [\vert C \vert]$ and $C_i, C_j \in C$.
\end{definition}
Using NPP also satisfies the requirement that any cluster is overall better than the succeeding clusters.
\begin{theorem}
\label{NPP arrangement theorem}
Let $C$ be the Nash Clustering of a two-player zero-sum symmetric game. Then $\text{NPP}(C_i, C_j) \geq 0$ $\forall i \leq j$, $C_i, C_j \in C$ and $\text{NPP}(C_i, C_j) \leq 0$ $\forall i > j$, $C_i, C_j \in C$.
\end{theorem}
\begin{figure*}[!t]
\vspace{-5pt}
    \centering
    \subfloat[]{\includegraphics[width = 0.33\textwidth]{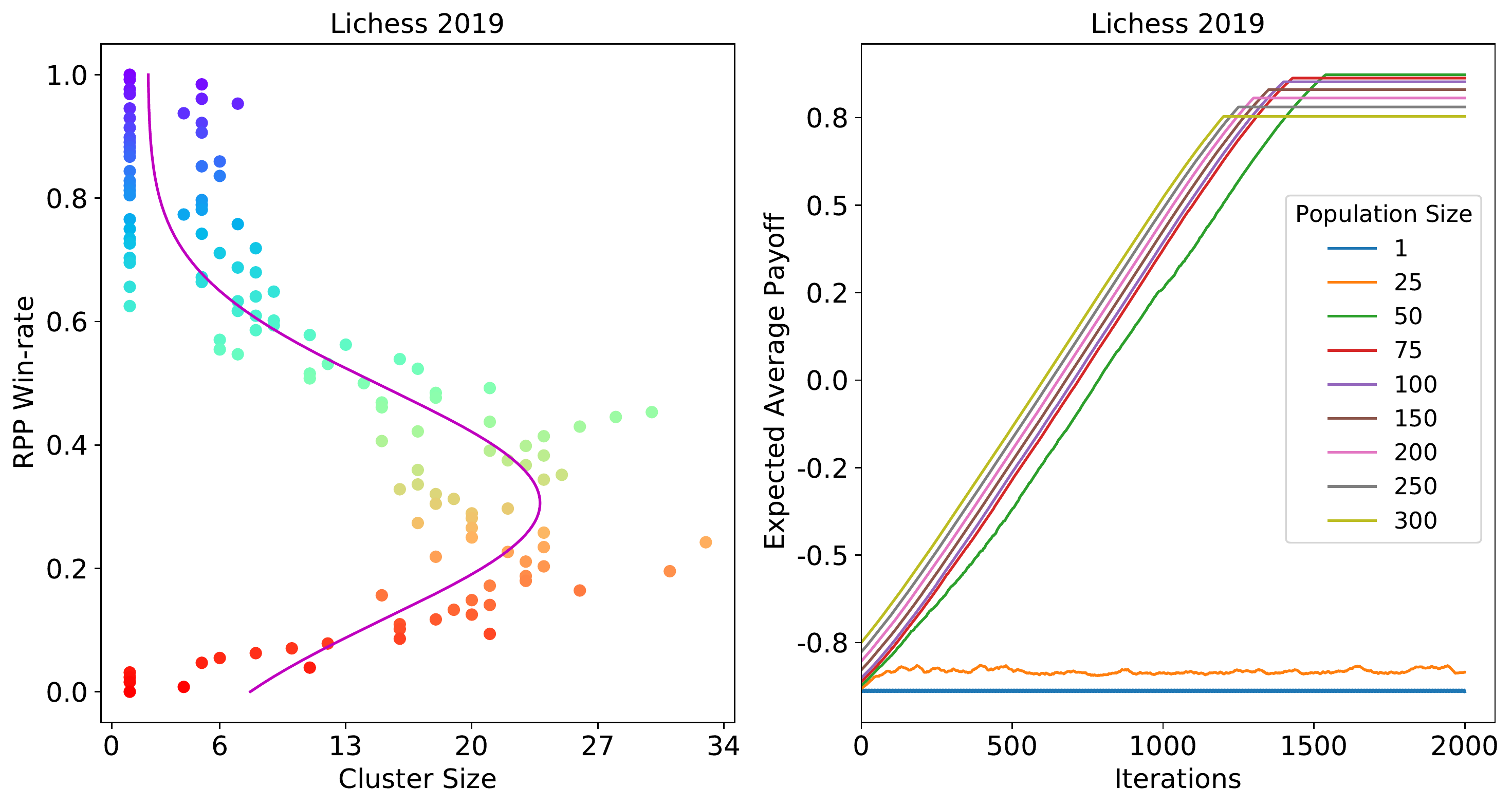}} %
    \subfloat[]{\includegraphics[width = 0.33\textwidth]{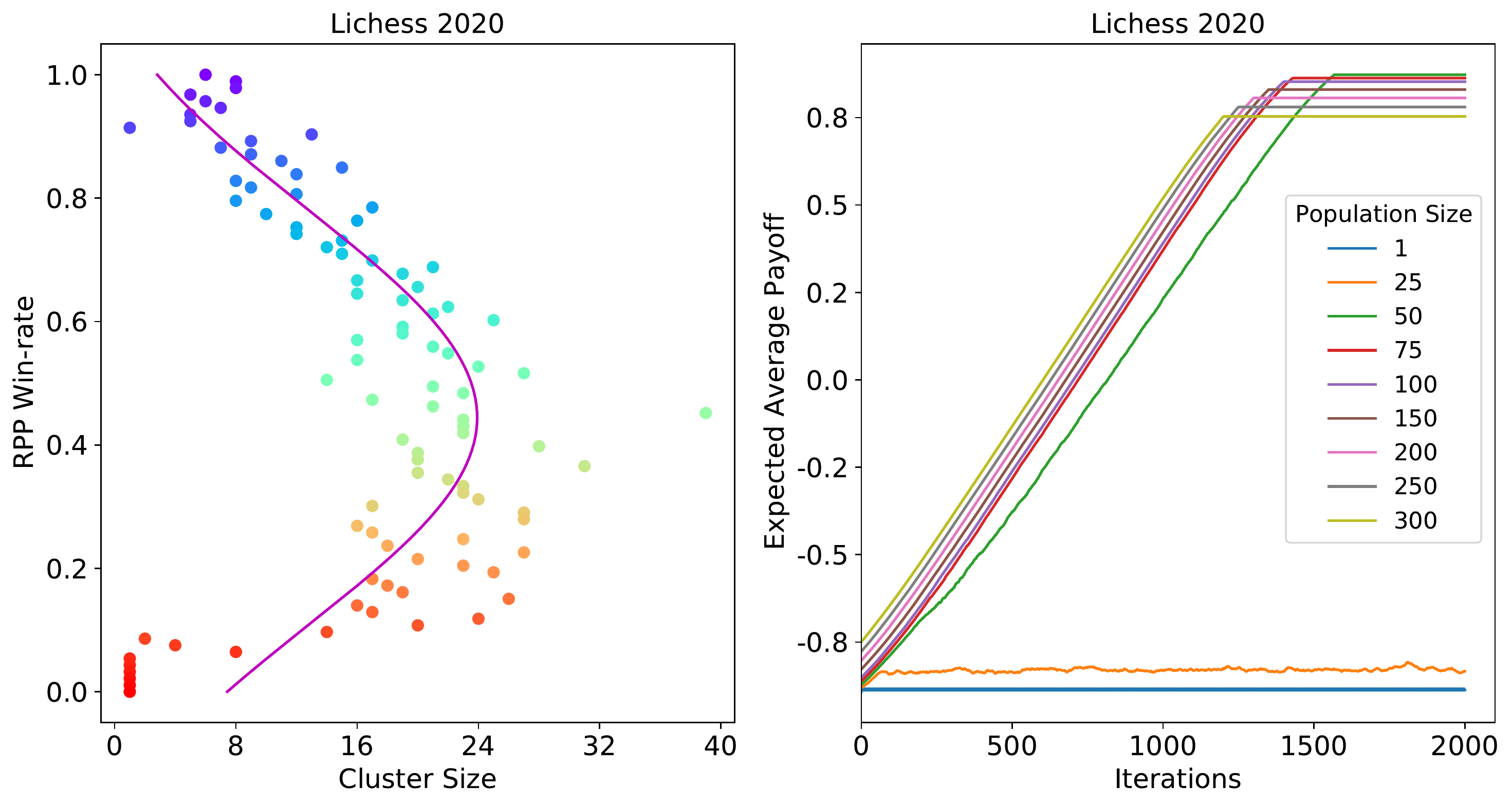}} %
    \subfloat[]{\includegraphics[width = 0.33\textwidth]{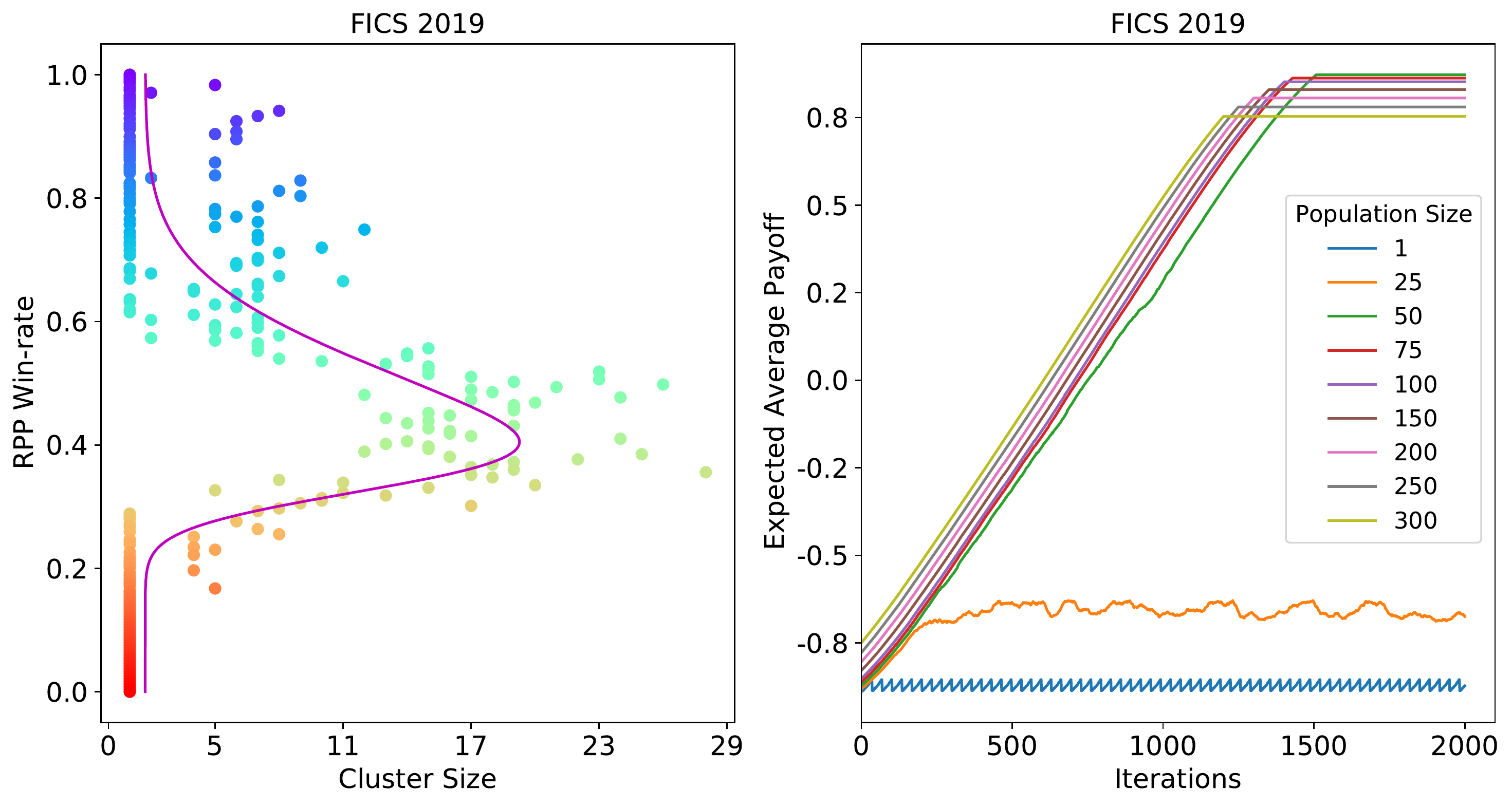}}
    \vspace{-10pt}
    \caption{(Left of each plot) Nash Clustering using RPP win-rate as the measure of transitive strength on Lichess 2019 (a), 2020 (b) and FICS 2019 (c). The spinning top geometry persists despite using a different measure of transitive strength compared to Figure \ref{fig: nash clustering and rps cycles elo}. (Right of each plot) Fixed-Memory Fictitious Play on Lichess 2019 (a), Lichess 2020 (b) and FICS 2019 (c). The training process does converge for all the above cases, provided that the starting population size is 50 or above. This is because the Nash cluster sizes are never much larger than 50, and therefore a population size of 50 ensures good coverage of the Nash clusters around any transitive strength. Such convergence behaviour also provides additional evidence that the spinning top geometry is present in the strategy space of real-world Chess. See Figure \ref{fig: appendix 2} of Appendix \ref{appendix: alternative measure results} and Figure \ref{fig: full results} of Appendix \ref{appendix: full results} for full versions of the results on Nash Clustering with RPP win-rate and fixed-memory Fictitious Play respectively.}
    \label{fig: nash clustering and training}
    \Description{Training process using Fixed-Memory Fictitious Play on Lichess games from 2019, 2020 and FICS games from 2019, compared with their corresponding non-transitivity plots using Nash Clustering.}
    \vspace{-5pt}
\end{figure*}

The proof of Theorem \ref{NPP arrangement theorem} is given in Appendix \ref{appendix: theorem proofs}. This theorem suggests that applying NPP to compute Equation \ref{RPP win-rate equation} instead of RPP would result in an identical win-rate for any Nash cluster, and the descending win-rate guarantee is preserved. NPP is more efficient since in Definition \ref{NPP definition}, $\mathbf{p_k}$ for any Nash cluster is the corresponding NE solved during Nash Clustering to obtain that cluster. However, this also means that NPP is more suitable as a relative performance measure of Nash clusters, and less so for general population of strategies. For the latter, RPP would be more suitable.

Figure \ref{fig: nash clustering and training} shows the results of applying Nash Clustering with win-rate as the measure of transitive strength. Again, the spinning top geometry persists for all cases, suggesting that Nash Clustering is agnostic to the measure of transitive strength.

\subsection{Rock-Paper-Scissor Cycles}


%

The second method to quantify non-transitivity is to measure the length of the longest cycle in the strategy space, since this stems directly from the definition of non-transitivity. A cycle is defined as the sequence of strategies that starts and ends with the same strategy, where any strategy beats the strategy next to it in the sequence. However, calculating the length of the longest cycle in a directed graph is well-known to be  NP-hard \cite{bjorklund2004approximating}. Therefore, we use the number of cycles of length three (i.e., Rock-Paper-Scissor or RPS cycles), passing through every pure strategy as a proxy to quantify non-transitivity. 

To obtain the number of RPS cycles passing through each pure strategy, an adjacency matrix is first formed from the payoff matrix. Letting the payoff matrix from Section \ref{subsection: payoff matrix construction} be $\boldsymbol{\mathcal{M}}$, the adjacency matrix $\boldsymbol{\mathcal{A}}$ can be written as $\big\{\boldsymbol{\mathcal{A}}_{i,j}\big\}_{i,j = 1}^{\vert S_g \vert}$ where $\boldsymbol{\mathcal{A}}_{i,j} = 1$ if $\boldsymbol{\mathcal{M}}_{i,j} > 0$, and $0$ otherwise.

\begin{theorem}
\label{closed walks theorem}
For an adjacency matrix $\boldsymbol{\mathcal{A}}$ where $\boldsymbol{\mathcal{A}}_{i,j} = 1$ if there exists a directed path from node $n_i$ to node $n_j$, the number of paths of length $k$ that starts from node $n_i$ and ends on node $n_j$ is given by $(\boldsymbol{\mathcal{A}}^k)_{i,j}$, where the length is defined as the number of edges.
\end{theorem}
The proof of Theorem \ref{closed walks theorem} can be found in Appendix \ref{appendix: theorem proofs}. By Theorem \ref{closed walks theorem}, the number of RPS cycles passing through a strategy can therefore be found by taking the diagonal of $\boldsymbol{\mathcal{A}}^3$. Note that we \textbf{cannot} go beyond three because the diagonal element of $\boldsymbol{\mathcal{A}}^n$ includes cycles with repeated nodes for $n > 3$.  
This completes the quantification of non-transitivity via counting RPS Cycles. The results are shown in Figure \ref{fig: nash clustering and rps cycles elo}. 
Finally, we also investigate the use of RPS cycles but using strategy win-rate defined in Equation \ref{strategy win-rate equation}, as a measure of transitive strength instead. 
Figure \ref{fig: appendix 2} of Appendix \ref{appendix: alternative measure results} shows the results of this method. Again, the spinning top geometry persists, showing that the RPS cycles method is agnostic towards the measure of transitive strength used.


\subsection{Connections to Existing Rating Systems}

As explained in Section \ref{section: related works} and Equation \ref{Elo formula}, there are various Chess player rating systems throughout the world, and each system is only valid for the pool of players over which the system is defined. Since the results in Figure \ref{fig: nash clustering and rps cycles elo} use game data from Lichess and FICS, they are not directly applicable to different rating systems such as USCF and FIDE. An investigation into the mapping between these rating systems is thus required.

We conduct a short experiment to obtain the mapping between the rating system of Lichess, USCF and FIDE. The dataset for this experiment is obtained from ChessGoals \cite{chessgoals_database}, which consists of player ratings in four sites, i.e., Chess.com \cite{chesscom_about}, FIDE, USCF, and Lichess. 
To investigate the mapping, 3 plots are first created, i.e., Lichess to USCF, Lichess to FIDE, and USCF to FIDE. A linear model is then fitted for each plot by MSE minimisation. The results are  shown in Figure \ref{fig: lr regression}. Since the mapping between these 3 rating systems are monotonic as demonstrated in Figure \ref{fig: lr regression}, the spinning top geometry observed in Figure \ref{fig: nash clustering and rps cycles elo} would still hold when mapped to USCF or FIDE rating systems. This is because the relative ordering of the Nash clusters or pure strategies based on Elo rating will be preserved. Therefore, our results can be  direct translated and cross-validated to different rating systems.



\section{Implications  for Learning} \label{section: implications of non-trans}

\begin{figure*}[!t]
    \centering
    \vspace{-15pt}
    \subfloat[]{\includegraphics[width = 0.33\textwidth]{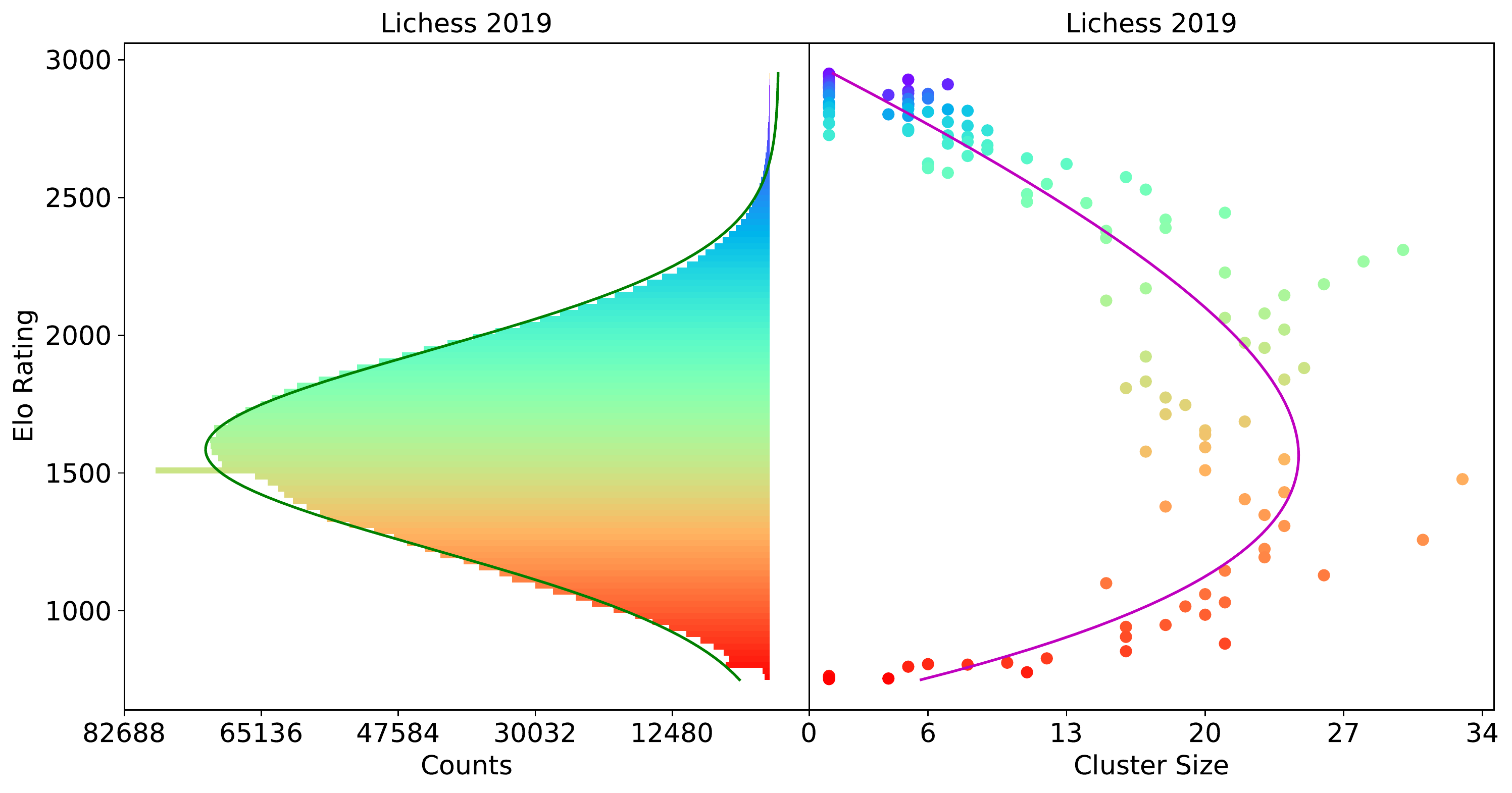}} %
    \subfloat[]{\includegraphics[width = 0.33\textwidth]{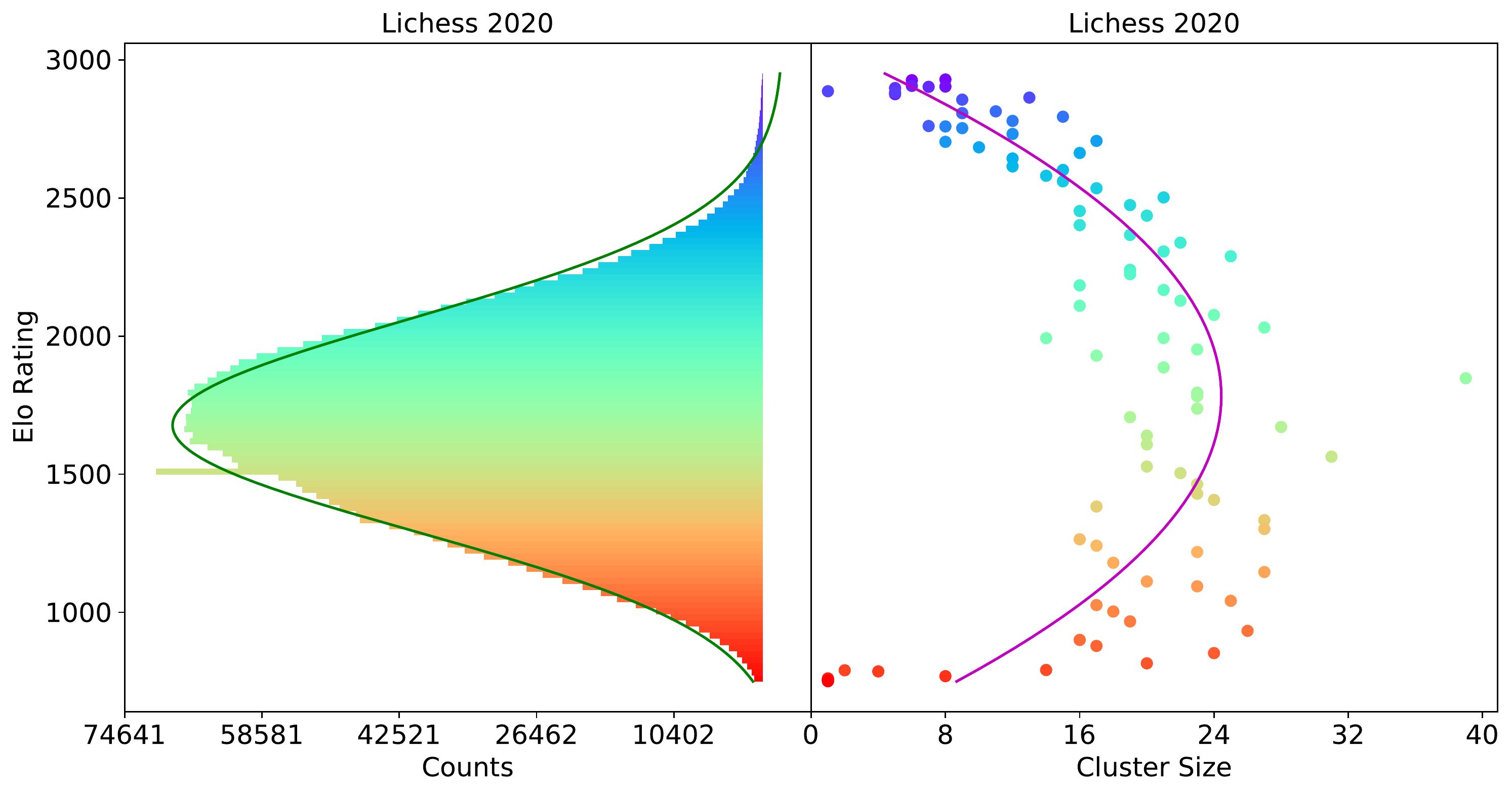}} %
    \subfloat[]{\includegraphics[width = 0.33\textwidth]{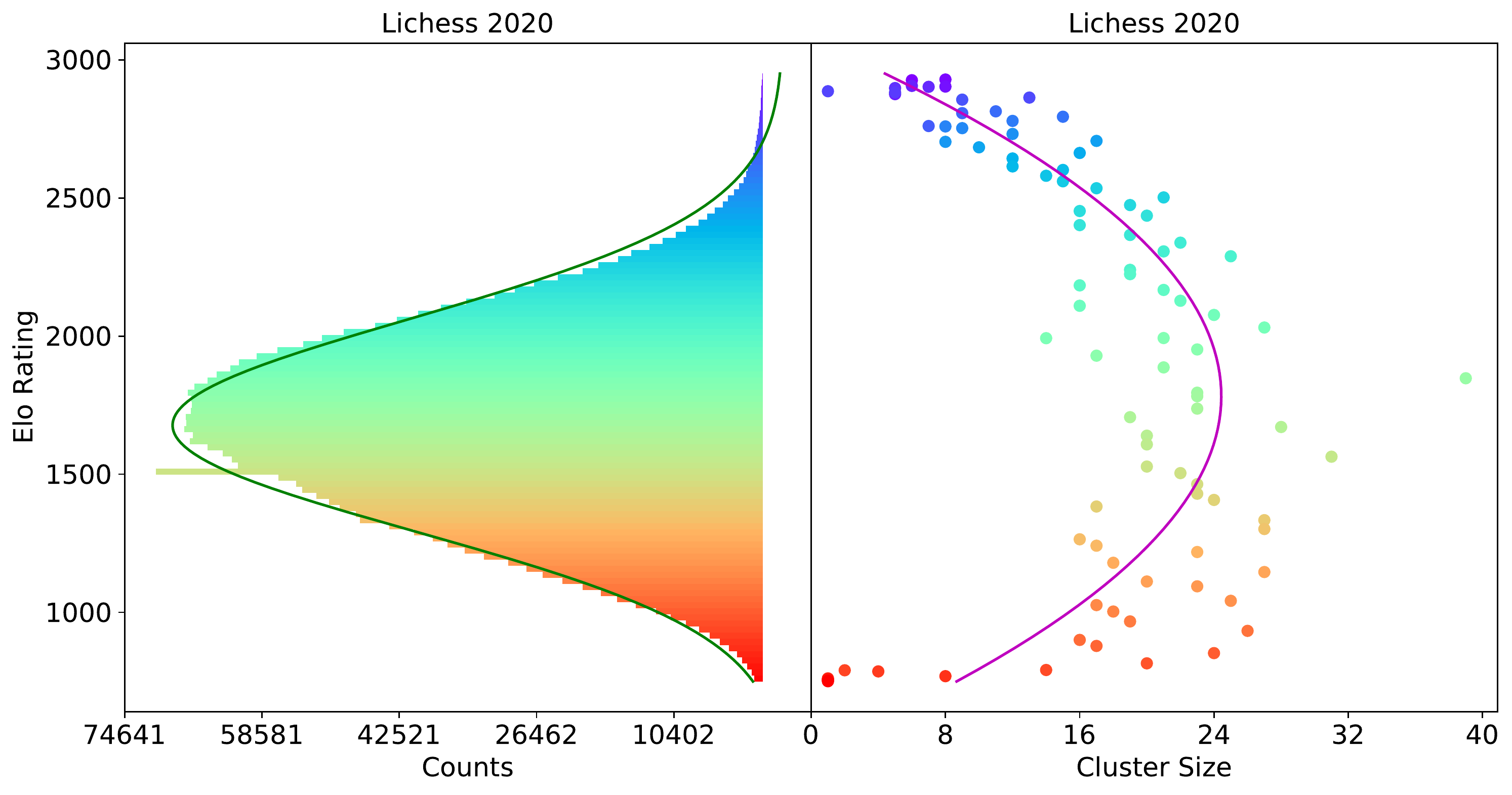}}
    \vspace{-10pt}
    \caption{(Left of each plot) Histogram of Player Elo Ratings and (Right of each plot) Nash Clustering on Lichess 2019 (a), Lichess 2020 (b) and FICS 2019 (c). See Figure \ref{fig: full results} of Appendix \ref{appendix: full results} for the full version of the results. The peak of the histogram as well as the bulk of the data consistently lies in the $1300$ to $1700$ Elo rating range, which tallies with the range of Elo rating in which players tend to get stuck as seen in \cite{chesscom_rating_stuck, lichess_rating_stuck, lichess_rating_stuck_2}, suggesting that players tend to get stuck in this range. Furthermore, the peaks of the histograms coincides with the peak of the fitted skewed normal curve on the Nash Clustering plots. This implies that high non-transitivity in the $1300$ to $1700$ rating range is one potential cause of players getting stuck in this range.}
    \label{fig: nash clustering and histogram}
    \Description{Histogram of player ratings from Lichess 2019 and 2020 games, as well as FICS 2019 games, compared with their corresponding non-transitivity plots using Nash Clustering.}
    \vspace{-5pt}
\end{figure*}

Based on the discovered geometry, we  investigate the impact of non-transitivity on multi-agent reinforcement learning \cite{yang2020overview}. 
Spcficially we design experiments to verify the influence of non-transitivity on a particular type of MARL algorithm: population-based learning \cite{lanctot2017unified,perez2021modelling,liu2021unifying,dinh2021online}, a technique that underpins many existing empirical success on developing game AIs \cite{zhou2021malib,lanctot2019openspiel}. 
  In fact, on simulated strategies \cite{czarnecki2020real}, it has been found that if the spinning top geometry exists in the strategy space, then the larger the non-transitivity of the game is in terms of layer size (for $k$-layered geometry) or Nash cluster size (for Nash Clustering), the larger is the population size required for methods such as fixed-memory fictitious play to  converge. We attempt to verify if the above conclusion still holds when real-world data is used instead. Furthermore, observing the above convergence trend with respect to the starting population size provides an even stronger evidence for the existence of the spinning top geometry in the strategy space of real-world Chess. 

As a proxy for population-based training method, fixed-memory Fictitious Play starts with a population (or set) of pure strategies $P^0$, of a fixed size $k$. The starting pure strategies in $P^0$ are picked to be the $k$ weakest pure strategies, based on the win-rate against other pure strategies, as shown in Equation \ref{strategy win-rate equation}. At every iteration $t$ of the fictitious play, the oldest pure strategy in the population $P^{t-1}$ is replaced with a pure strategy that is not in $P^{t-1}$. The replacement strategy is required to beat all the strategies in the current population on average, i.e., suppose that $\pi$ is the replacement strategy at iteration $t$. Then $\sum_{s \in P^{t-1}} M(\pi, s) > 0$. If there are multiple strategies satisfying this condition, then the selected strategy is the one with the lowest win-rate, i.e., $\pi = \argmin_{s \in S_a} TS(s)$, where $S_a = \{s \in S_g \backslash P^{t-1}, \sum_{r \in P^{t-1}} M(s, r) > 0\}$.  

To measure the  performance at every iteration of the fictitious play, we use the expected average payoff of the population at that iteration. Suppose at any iteration $t$, the probability allocation over population $P^t$, which is of size $k$, is $\mathbf{p^t}$. Given the payoff matrix $\boldsymbol{\mathcal{M}}$ is of size $m \times m$, the performance of $WR(P^t)$ is written as 
\begin{equation}
WR(P^t) = \sum_{i \in [k]} \frac{\mathbf{p^t}_i}{m} \sum_{j \in [m]} M(P^t_i, S_j), 
\end{equation}
where $S_j$ is the pure strategy corresponding to the $j^{th}$ column of $\boldsymbol{\mathcal{M}}$. Fixed-memory fictitious play is then carried out for a number of starting population sizes. The results are shown in Figure \ref{fig: nash clustering and training}. 


However, the minimum population size needed for training to converge does not always correspond to the largest Nash cluster size, and there are  two contributing factors. The first factor is that covering a full Nash cluster, i.e., having a full Nash cluster in the population at some iteration, only guarantees the replacement strategy at that iteration to be from a cluster of greater transitive strength than the covered cluster. There are no guarantees that this causes subsequent iterations to cover Nash clusters of increasing transitive strength. The second factor is that the convergence guarantee with respect to the $k$-layered game of skill geometry defined in Definition \ref{definition: k-layered}. Since Nash Clustering is a relaxation of this geometry, covering a full Nash cluster does not guarantee that the  minimal population size requirement will be satisfied.


\section{Link to Human Skill Progression}
Section \ref{section: implications of non-trans} demonstrates that non-transitivity has implications on the performance improvement of AI agents during training. On the other hand, one of the most common issue faced by Chess players is the phenomenon of being stuck at a certain Elo rating \cite{chesscom_rating_stuck, lichess_rating_stuck, lichess_rating_stuck_2}. Having performed experiments using real-data, we also investigate whether non-transitivity has implications on human skill or rating progression, in particular, whether there exists a link between the non-transitivity of the strategy space of real-world Chess to the phenomenon of players getting stuck in certain Elo ratings. This is conducted by comparing the histogram of player ratings at every year using the data  from Section \ref{subsection: payoff matrix construction}, with the corresponding Nash Clustering plot. The results are shown in Figure \ref{fig: nash clustering and histogram}.

A plausible explanation is that in the $1300$-$1700$ Elo rating range, there exists a lot of RPS cycles from Figure \ref{fig: nash clustering and rps cycles elo}, which means there exists large number of counter strategies. Therefore, to improve from this Elo rating range, players would need to learn how to win against many strategies, for example, by dedicating more efforts to learn  diverse opening tactics. Otherwise, their play-style would be at risk of getting countered by some strategies at that level. This is in contrast to ratings where non-transitivity is low. With a lower number of counter strategies, it is arguably easier to improve, since there are relatively fewer strategies to adapt towards.

\section{Conclusion}
The goal of this paper was to measure the non-transitivity of Chess games, and investigate its potential implications for training effective AI agents, as well as on human skill progression.  
Our work is novel in the sense that we, for the first time, profile the geometry of the strategy space based on the real-world data from human players; this  
in turn prompts the discovery of the spinning top geometry that was previously only observed  on AI-generated strategies, and more importantly,  the discovery of the relationship between the degree of non-transitivity and  the  progression of human skills in playing Chess. 
Additionally, the discovered relationship also indicates that in order to overcome the non-transitivity in the strategy space, it is imperative to maintain large population sizes and enforcing diversity when applying population-based  methods in training AIs. 
Throughout our analysis, we managed to tackle several practical challenges. For example, despite the variations of Elo rating systems for Chess, we introduced effective mappings between those rating systems on over one billion match records; this enabled direct translations and cross-validations of  our findings  across different rating systems.
Finally, we argue that although we only study Chess in this paper, our conclusions  should also hold  for other complex zero-sum  games such as Go, DOTA and StarCraft; we leave the verification on real-world data for future work. 
\clearpage
\bibliographystyle{ACM-Reference-Format}
\bibliography{sample-base}


\clearpage

\appendix

\onecolumn

\FloatBarrier
\section{Additional Results} \label{appendix: full results}
\begin{figure}[H]
    \centering
    \includegraphics[width = 0.245\textwidth, height = 0.42 \textheight]{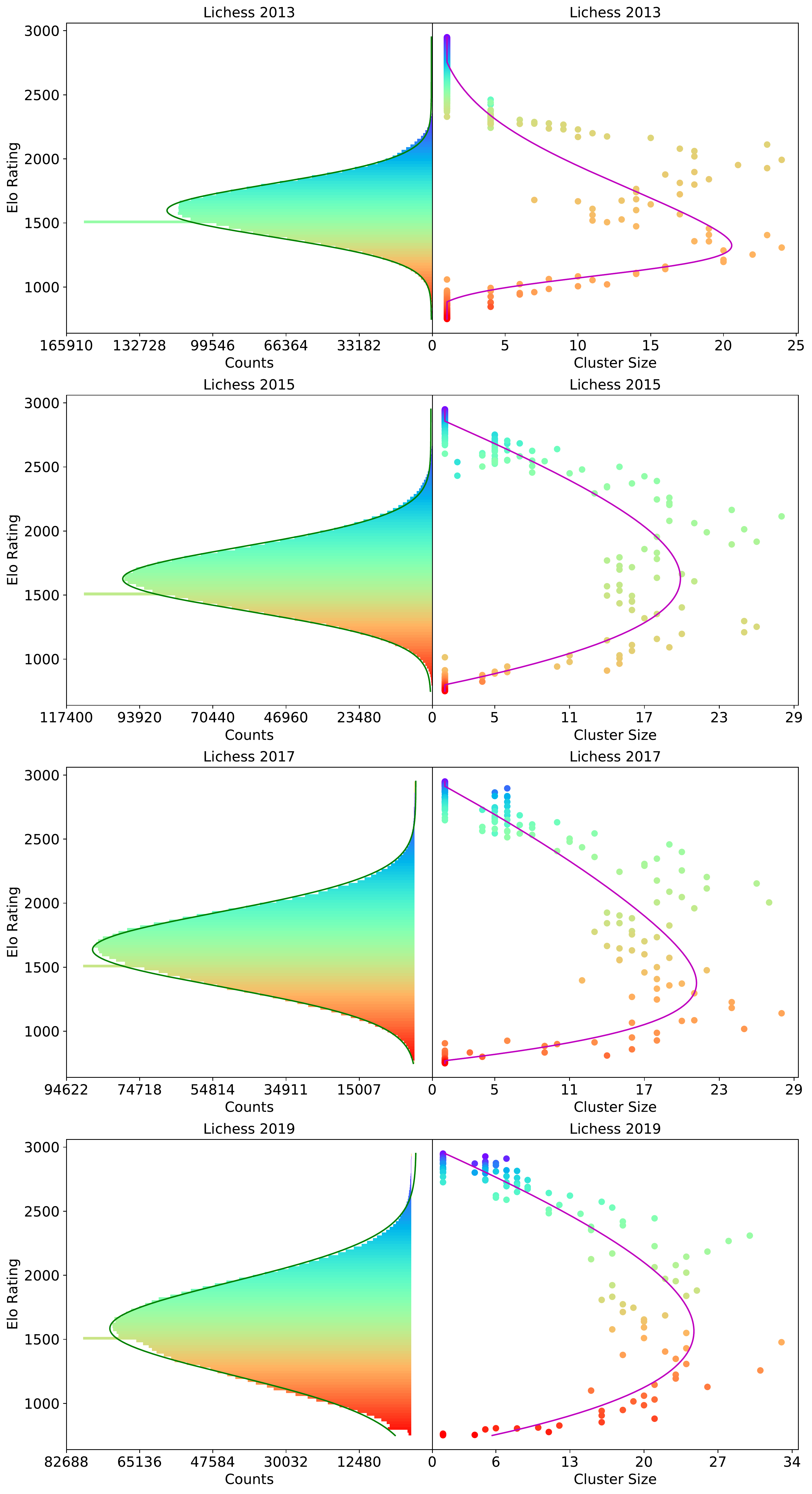} %
    \includegraphics[width = 0.1225\textwidth, height = 0.42 \textheight]{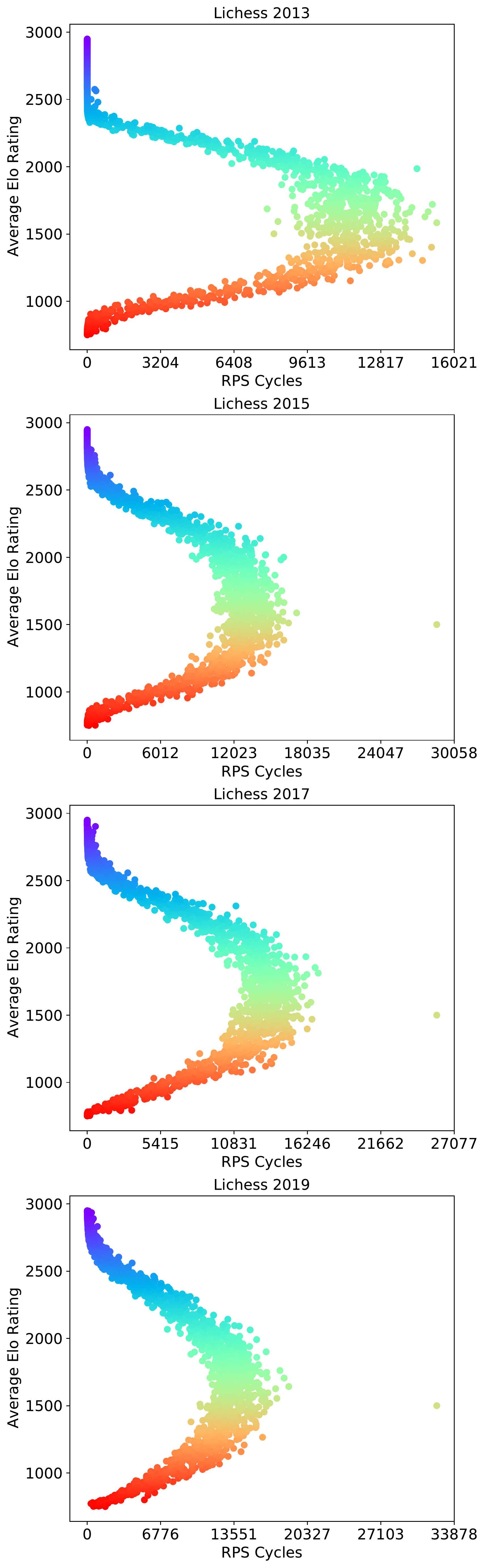} %
    \includegraphics[width = 0.1225\textwidth, height = 0.42 \textheight]{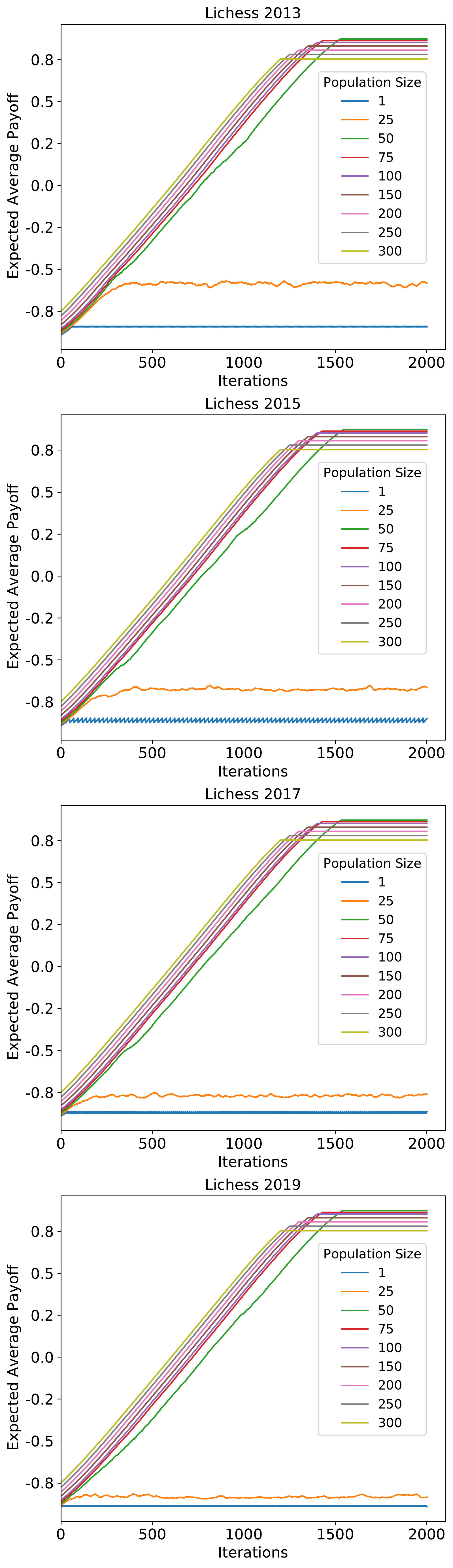}
    \includegraphics[width = 0.245\textwidth, height = 0.42 \textheight]{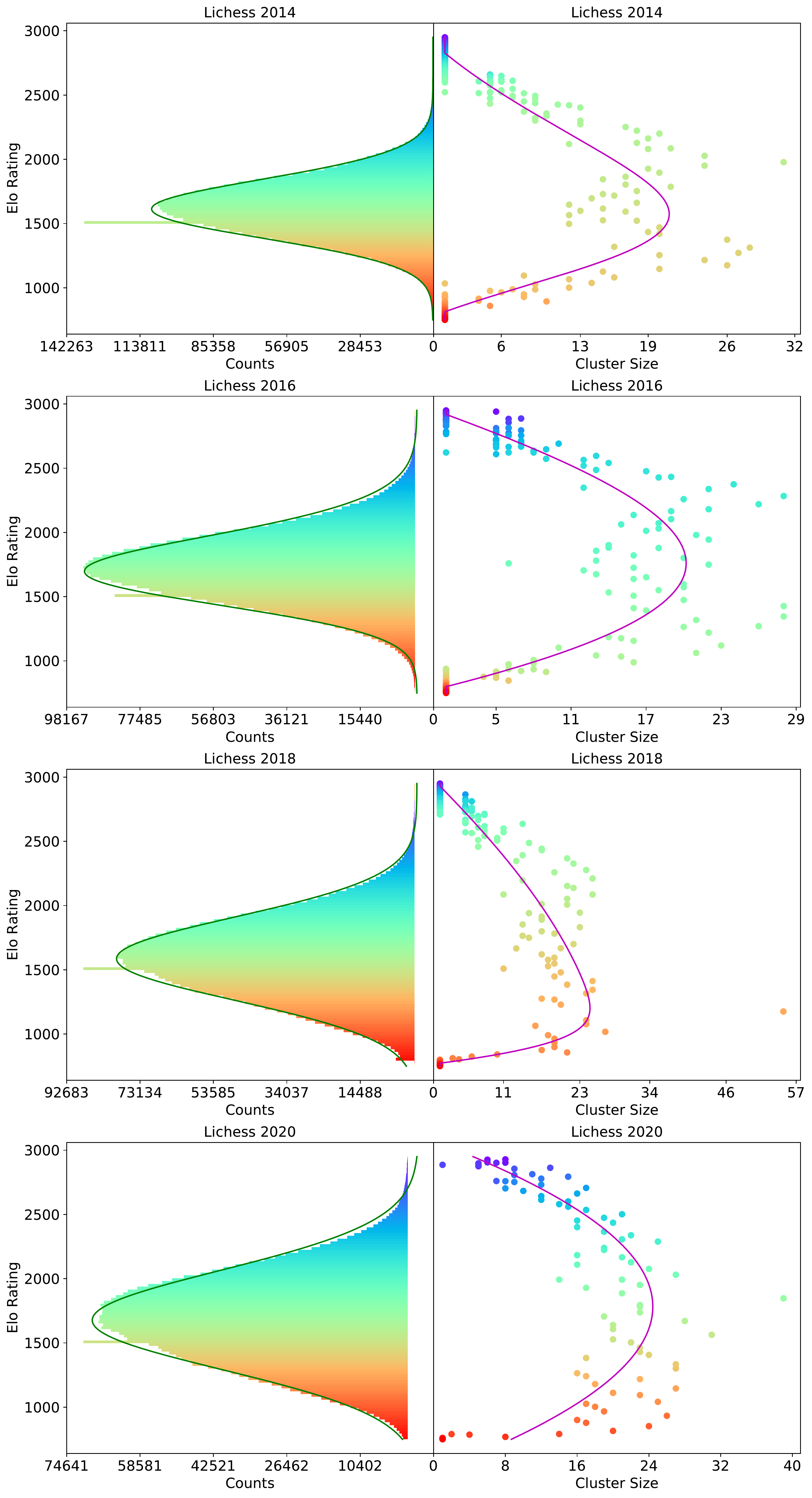} %
    \includegraphics[width = 0.1225\textwidth, height = 0.42 \textheight]{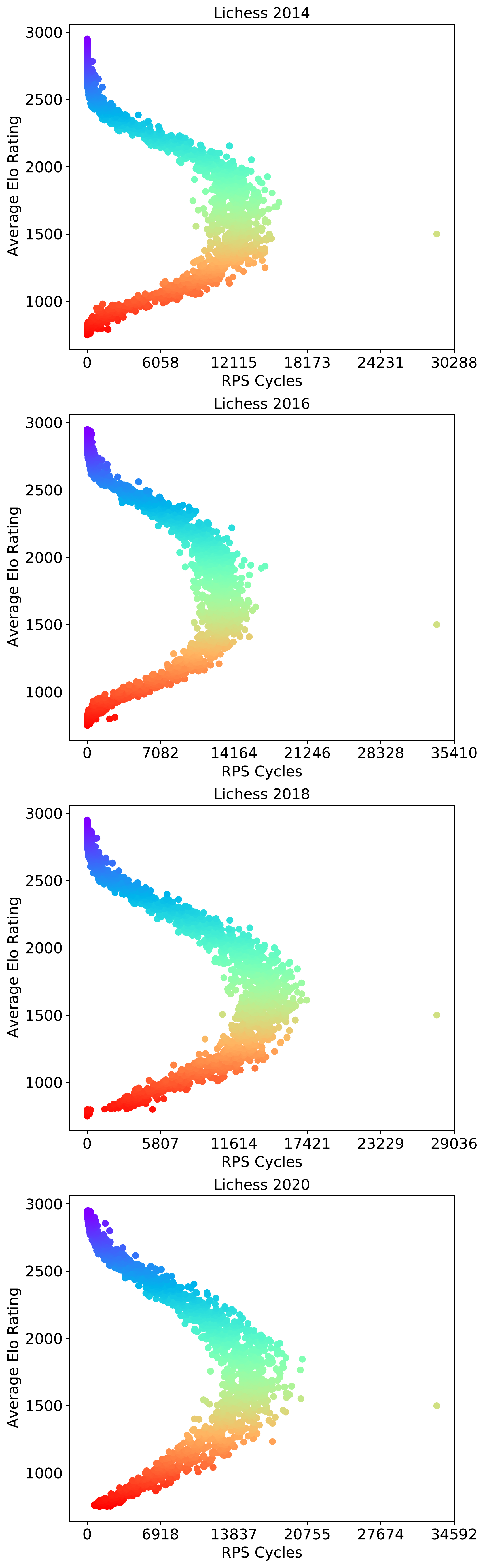} %
    \includegraphics[width = 0.1225\textwidth, height = 0.42 \textheight]{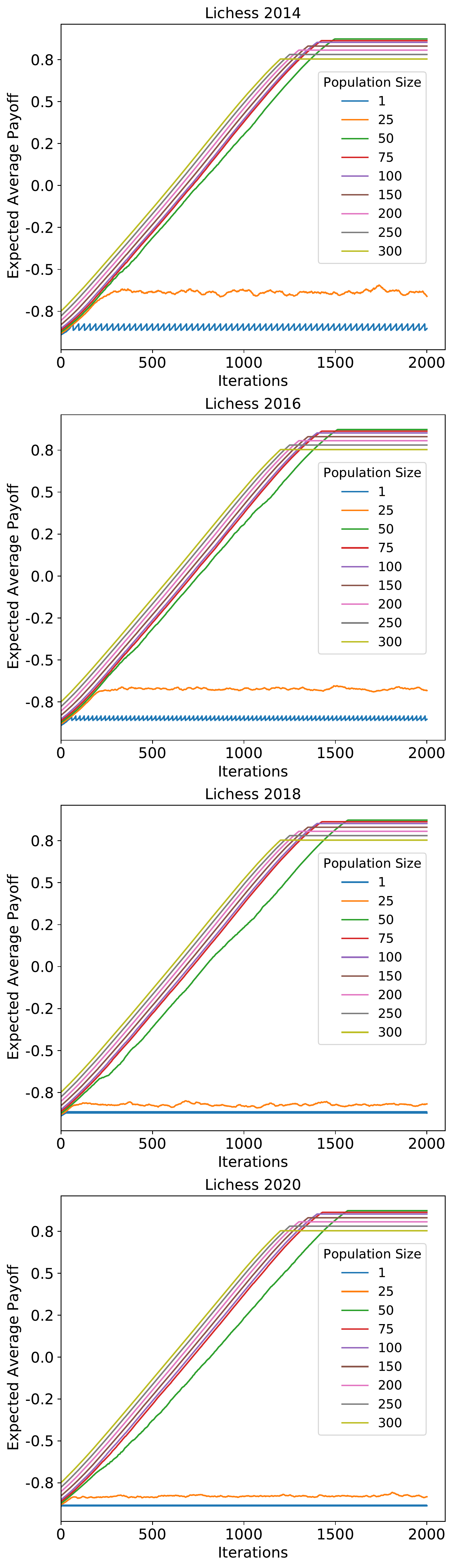}
    
    \includegraphics[width = 0.245\textwidth, height = 0.105 \textheight]{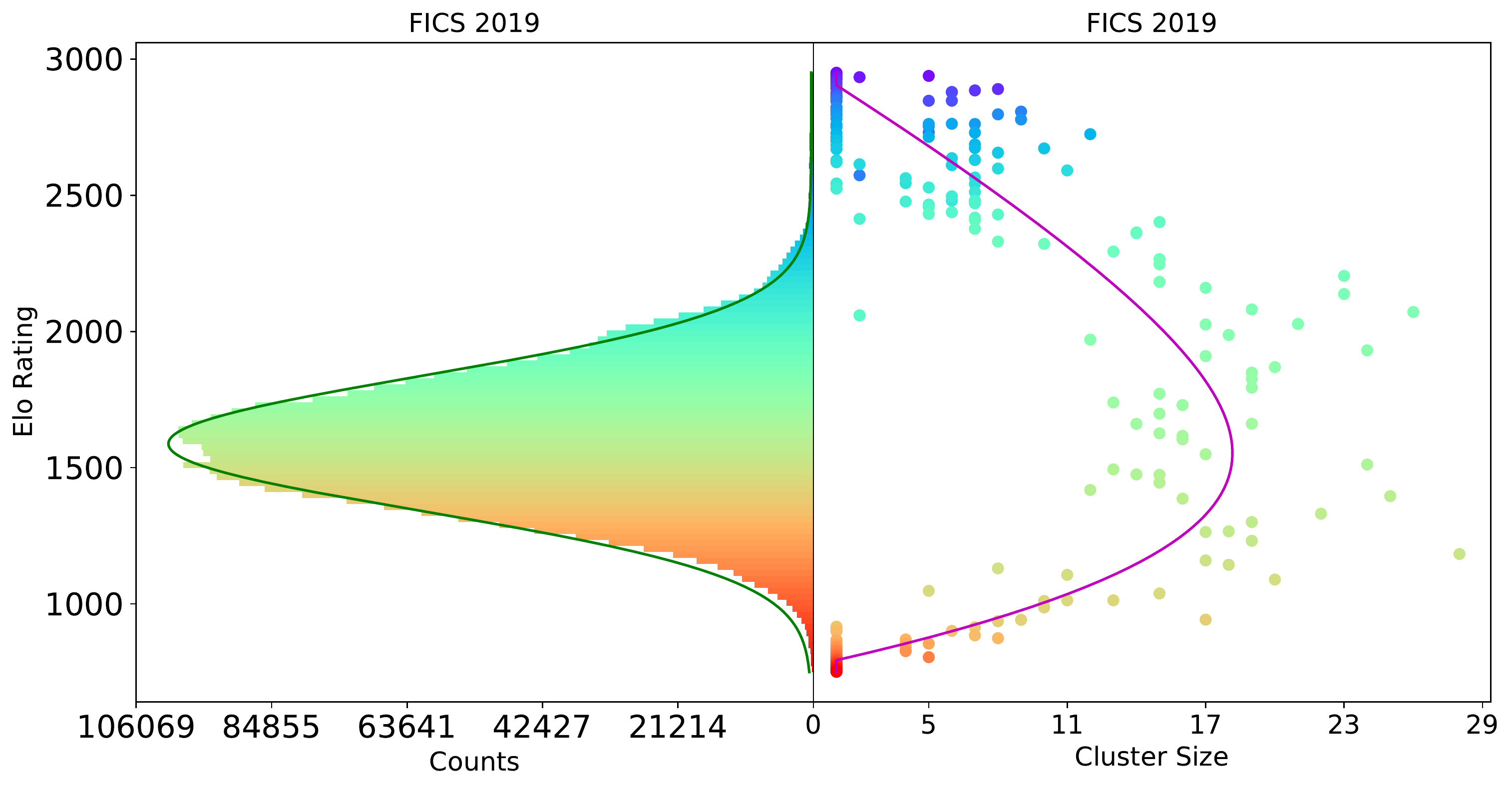} %
    \includegraphics[width = 0.1225\textwidth, height = 0.105 \textheight]{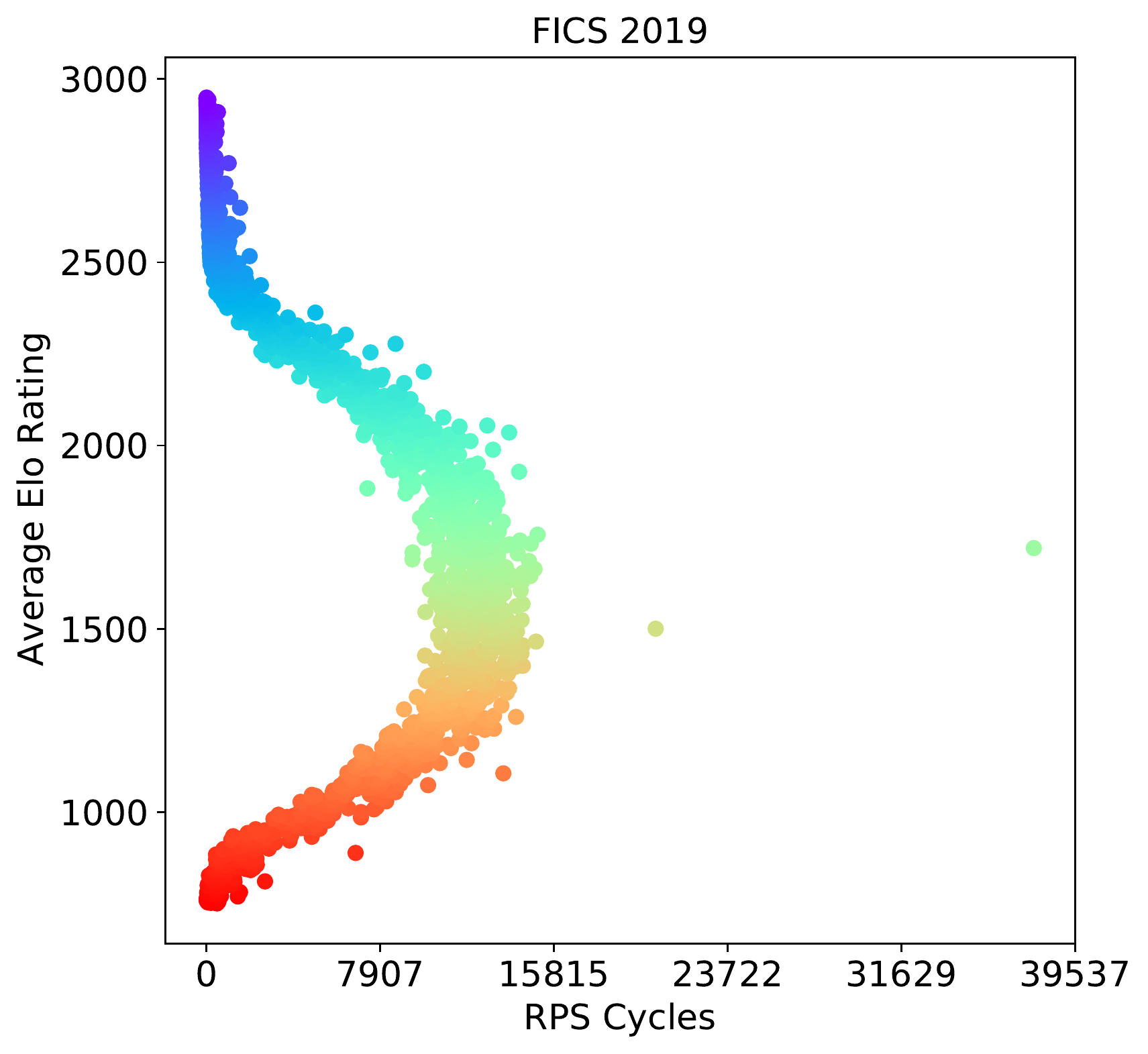} %
    \includegraphics[width = 0.1225\textwidth, height = 0.105 \textheight]{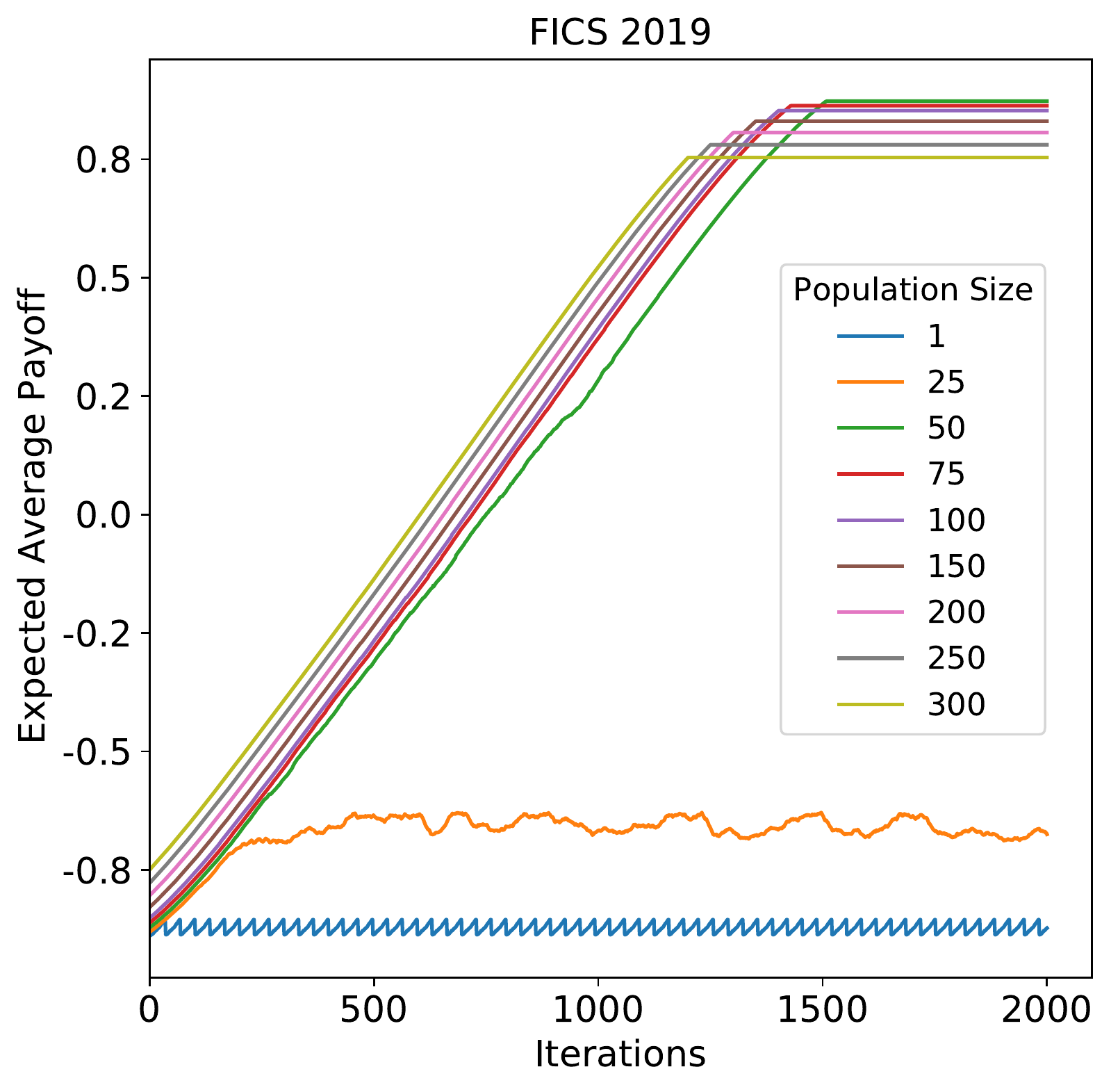}

    \caption{Histogram of player Elo ratings (first of each plot), Nash Clustering (second of each plot), RPS Cycles (third of each plot), and fixed-memory Fictitious Play (fourth of each plot) on Lichess 2013-2020 and FICS 2019 data.} 
    \Description{Histogram of player ratings, non-transitivity plots using RPS cycles counting and Nash Clustering, and training performance using Fixed-Memory Fictitious Play with various population sizes, conducted on games from Lichess 2013-2020 and FICS 2019.}
    \label{fig: full results}
\end{figure}
\FloatBarrier

\section{Two-Staged Sampling Algorithm} \label{appendix: two-staged sampling}
\begin{algorithm}[H]
\SetAlgoLined
\textbf{Inputs}: Set of $m$ objects, i.e., $U = \{U_1, ..., U_m\}$, number of objects to sample $d \in \mathbb{N}_1$; \\
Divide $U$ into $k$ chunks, i.e., $H_1, ..., H_k$, each of size $h$, i.e., $h \times k = m$, $h,k \in \mathbb{N}_1$, $H_i \cap H_j = \emptyset$ $\forall i \neq j$, $i,j \in [k]$, $\bigcup_{j \in [k]} H_j = U$; \\
Initialize $G = []$;\\
\For{$i \in [k]$}{
    Sample $d$ objects uniformly from $H_i$, forming $t$; \\
    Append all elements in $t$ to $G$;
}
Sample $d$ objects uniformly from $G$, forming $F$; \\
Output F as the set of $d$ objects sampled uniformly from $U$;
\caption{Two-Stage Uniform Sampling}
\label{algo: two-stage uniform sampling}
\end{algorithm}
\FloatBarrier

\noindent Algorithm \ref{algo: two-stage uniform sampling} can be proven to sample uniformly from $U$, i.e., the probability of any objects in $U$ being sampled into $F$ is $\frac{d}{m}$

\begin{proof}
For any object in $U$ to end up in $F$, they must first get picked into $G$. Since $H_i \cap H_j = \emptyset$ $\forall i,j \in [k]$, $i \neq j$, the probability that any object gets sampled into $G$ is $\frac{\Mycomb[h-1]{d-1}}{\Mycomb[h]{d}} = \frac{d}{h}$. Then given they are in $G$, they must get sampled into $F$. The probability that this happens is $\frac{\Mycomb[d \times k - 1]{d-1}}{\Mycomb[d \times k]{d}} = \frac{1}{k}$. Hence the probability that any object from $U$ gets sampled into $F$ is $\frac{d}{h \times k} = \frac{d}{m}$.
\end{proof}

\clearpage
\section{Results with Alternative Measure of Transitive Strength} \label{appendix: alternative measure results}

\FloatBarrier
\begin{figure}[H]
\vspace{-7.5pt} 
    \centering
    \includegraphics[width = 0.33\textwidth]{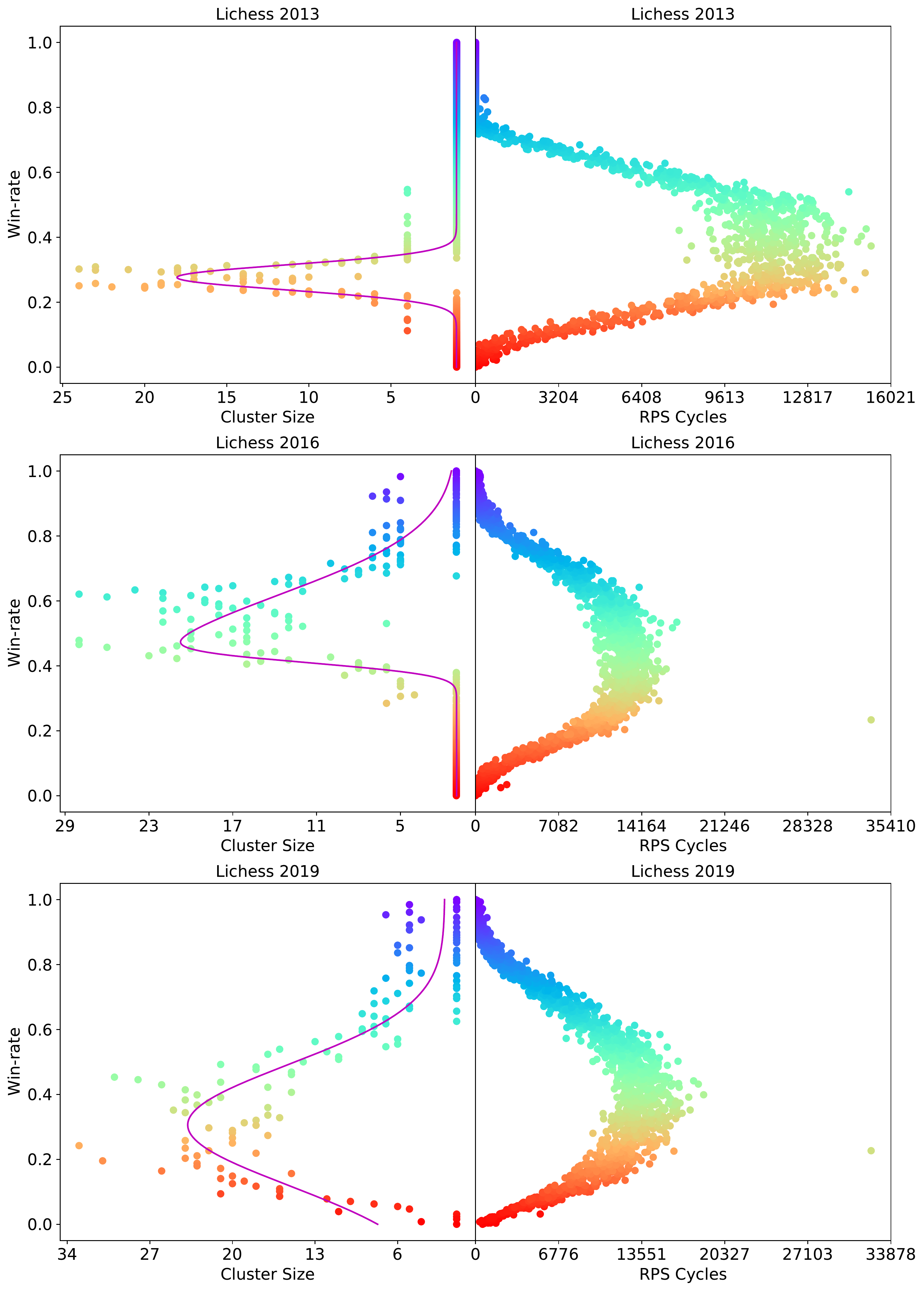} %
    \includegraphics[width = 0.33\textwidth]{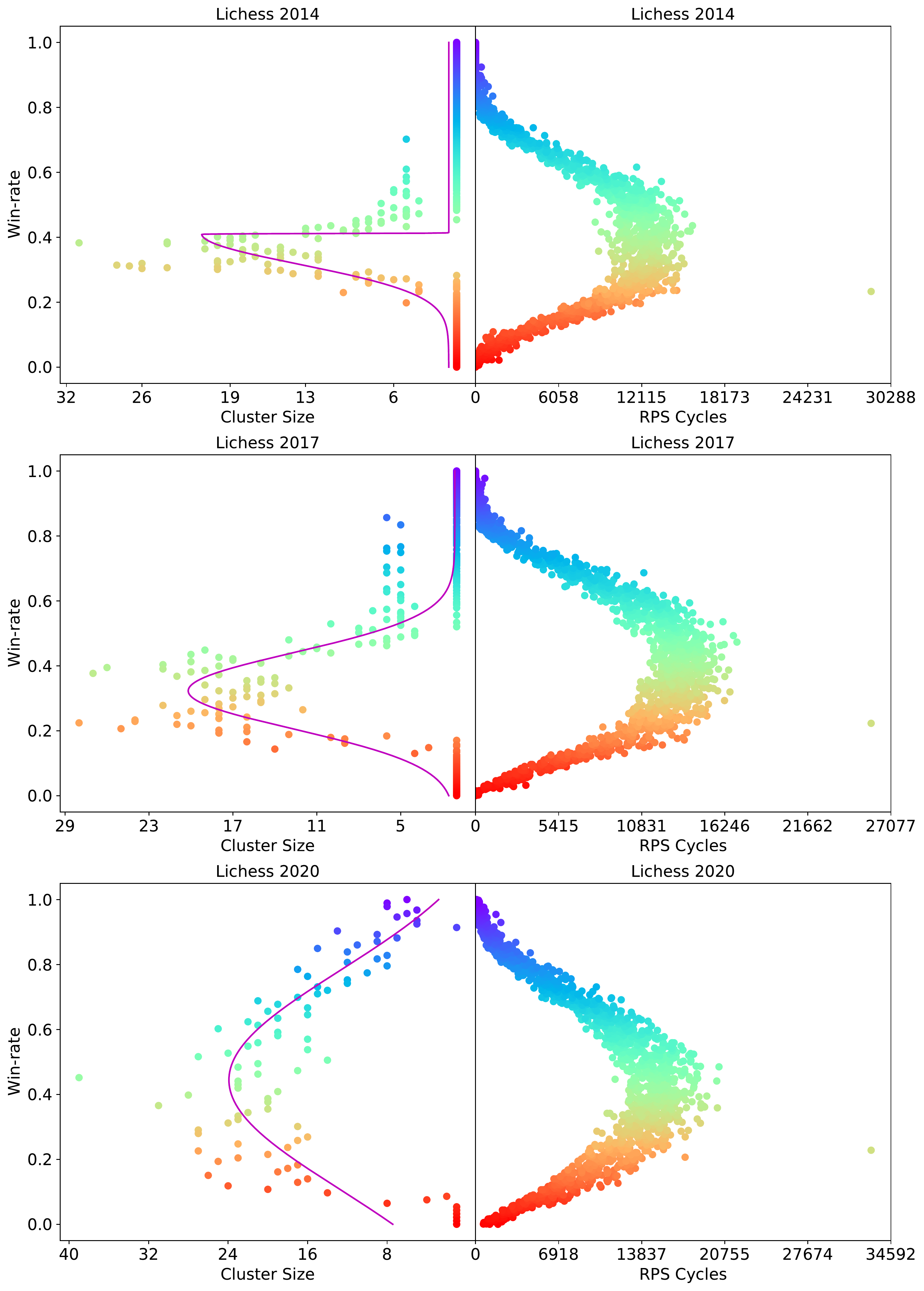} %
    \includegraphics[width = 0.33\textwidth]{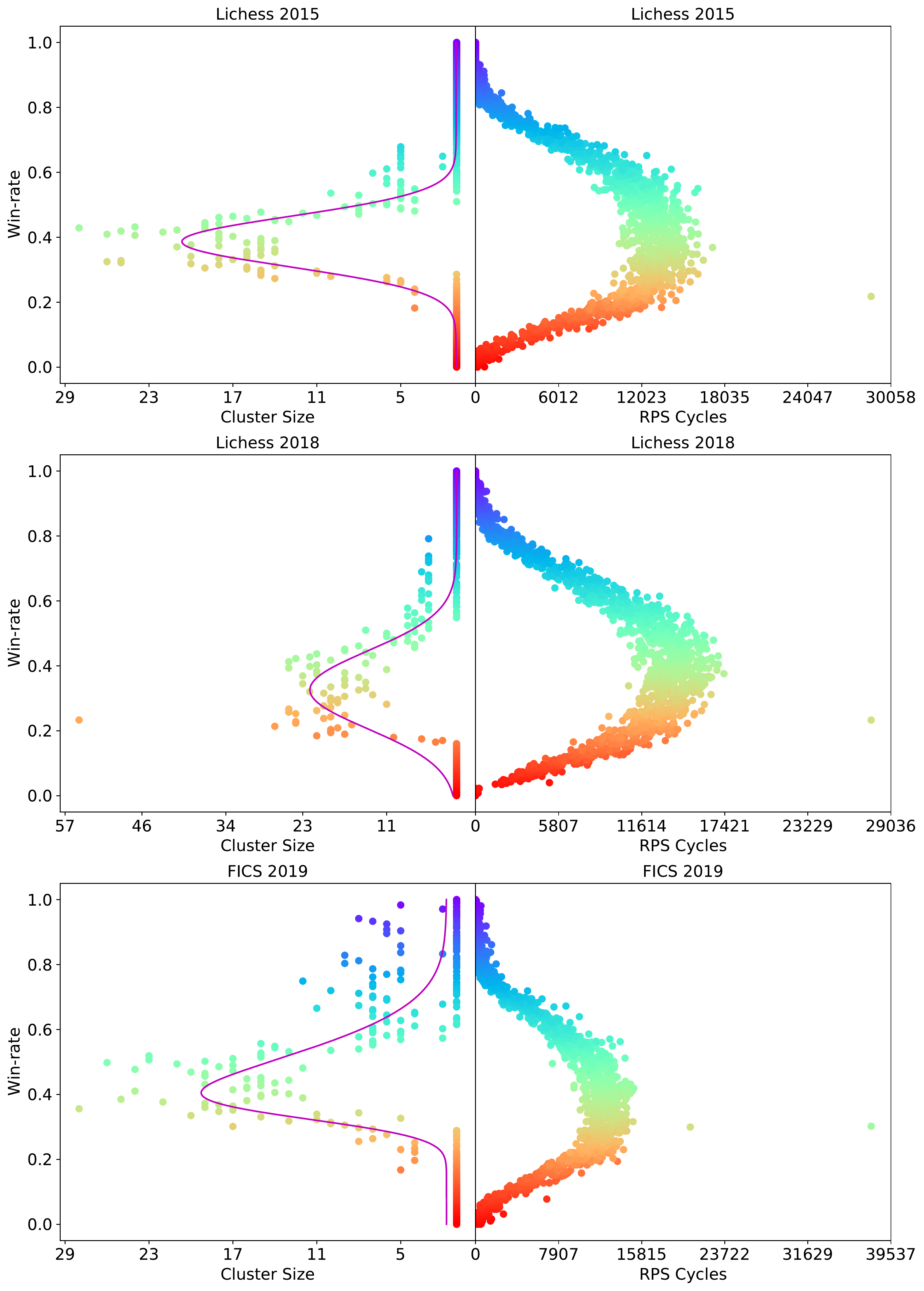}
    \vspace{-20pt}
    \caption{Measurements of Nash Clustering (left of each plot) and RPS Cycles (right of each plot) using win-rate as measure of transitive strength on Lichess 2013-2020 and FICS 2019 data.}
    \label{fig: appendix 2}
    \vspace{-10pt}
    \Description{Nash clustering and RPS cycles using win-rate as the measure of transitive strength and conducted on Lichess 2013-2020 games as well as FICS 2019 games.}
\end{figure}
\FloatBarrier


\section{Theorem Proofs} \label{appendix: theorem proofs}
\begin{theorem}
Let $C$ be the Nash Clustering of a two-player zero-sum symmetric game. Then $\text{NPP}(C_i, C_j) \geq 0$ $\forall i \leq j$, $C_i, C_j \in C$ and $\text{NPP}(C_i, C_j) \leq 0$ $\forall i > j$, $C_i, C_j \in C$.
\end{theorem}
\begin{proof}
Let $C$ be the Nash Clustering of a two-player zero-sum symmetric game with a payoff matrix of $\boldsymbol{\mathcal{M}}$ and strategy space $S_g$ for each of the two players. For any $k \in [\vert C \vert]$, let $\mathbf{p_k}$ be the NE solved to obtain the $k^{th}$ cluster, i.e., $\mathbf{p_k} = \text{Nash}(\boldsymbol{\mathcal{M}} \vert \bigcup_{i = k}^{\vert C \vert} C_i)$ where $\mathbf{p_k} \in \Delta_{\vert C \vert}$. Since $\mathbf{p_k}$ is restricted to exclude $\bigcup_{i=1}^{k-1} C_i$, therefore $\mathbf{p_k}$ would be all zeroes at the positions corresponding to all pure strategies $s \in \bigcup_{i = 1}^{k-1} C_i$. Let $X_k = \bigcup_{i = 1}^{k-1} C_i$ and $\boldsymbol{\mathcal{M}}_{-X_k}$ be the payoff matrix of the game but excluding the rows and columns that corresponds to the pure strategies in $X_k$. Also let $\mathbf{p_k}^{X_k}$ denote the vector $\mathbf{p_k}$ but removing all entries at positions corresponding to the pure strategies in $X_k$. Hence $\boldsymbol{\mathcal{M}}_{-X_k} \in \mathbb{R}^{\vert S_g \backslash X_k \vert \times \vert S_g \backslash X_k \vert}$ and $\mathbf{p_k}^{X_k} \in \Delta_{\vert S_g \backslash X_k \vert}$. Furthermore, it is also true that $\mathbf{p}^\top \boldsymbol{\mathcal{M}} \mathbf{p} = (\mathbf{p}^{X_k}) ^\top \boldsymbol{\mathcal{M}}_{-X_k} \mathbf{p}^{X_k}$ for any $\mathbf{p} \in \Delta_{\vert S_g \vert}$ provided that the entries of $\mathbf{p}$ at positions corresponding to pure strategies in $X_k$, are zeroes. \\
By definition of NE in \cite{daskalakis_lec02}, it is true that $(\mathbf{p_k}^{X_k}) ^\top \boldsymbol{\mathcal{M}}_{-X_k} \mathbf{p_k}^{X_k} \leq (\mathbf{p_k}^{X_k}) ^\top \boldsymbol{\mathcal{M}}_{-X_k} \mathbf{p'} \ \forall \mathbf{p'} \in \Delta_{\vert S_g \backslash X_k \vert}$ and $(\mathbf{p_k}^{X_k}) ^\top \boldsymbol{\mathcal{M}}_{-X_k} \mathbf{p_k}^{X_k} \geq \mathbf{p'}^\top \boldsymbol{\mathcal{M}}_{-X_k} \mathbf{p_k}^{X_k}$ $\forall \mathbf{p'} \in \Delta_{\vert S_g \backslash X_k \vert}$.
Now consider two clusters $C_i$ and $C_j$ where $i < j$ and $C_i, C_j \in C$. By the first definition of NE earlier, 
$(\mathbf{p_i}^{X_i} )^\top \boldsymbol{\mathcal{M}}_{-X_i} \mathbf{p_i}^{X_i} \leq  (\mathbf{p_i}^{X_i} )^\top \boldsymbol{\mathcal{M}}_{-X_i} \mathbf{p'} \ \forall \mathbf{p'} \in \Delta_{\vert S_g \backslash X_i \vert}$
which implies $(\mathbf{p_i}^{X_i} )^\top \boldsymbol{\mathcal{M}}_{-X_i} \mathbf{p_i}^{X_i} \leq (\mathbf{p_i}^{X_i} )^\top \boldsymbol{\mathcal{M}}_{-X_i} \mathbf{p_j}^{X_i}$ and thus equivalently $\mathbf{p_i}^\top \boldsymbol{\mathcal{M}} \mathbf{p_i} \leq \mathbf{p_i}^\top \boldsymbol{\mathcal{M}} \mathbf{p_j}$. Then by property of skew-symmetric matrices, $\mathbf{p_i}^\top \boldsymbol{\mathcal{M}} \mathbf{p_i} = 0$, hence $\mathbf{p_i}^\top \boldsymbol{\mathcal{M}} \mathbf{p_j} \geq 0$, i.e., $\text{NPP}(C_i, C_j) \geq 0$. Furthermore, by the second definition of NE earlier $(\mathbf{p_i}^{X_i})^\top \boldsymbol{\mathcal{M}}_{-X_i} \mathbf{p_i}^{X_i} \geq \mathbf{p'}^\top \boldsymbol{\mathcal{M}}_{-X_i} \mathbf{p_i}^{X_i}\ \forall \mathbf{p'} \in \Delta_{\vert S_g \backslash X_i \vert}$, which implies $(\mathbf{p_i}^{X_i})^\top \boldsymbol{\mathcal{M}}_{-X_i} \mathbf{p_i}^{X_i} \geq (\mathbf{p_j}^{X_i})^\top \boldsymbol{\mathcal{M}}_{-X_i} \mathbf{p_i}^{X_i}$ and thus equivalently $\mathbf{p_i}^\top \boldsymbol{\mathcal{M}} \mathbf{p_i} \geq \mathbf{p_j}^\top \boldsymbol{\mathcal{M}} \mathbf{p_i}$. Finally, using the property of skew-symmetric matrices again, $\mathbf{p_j}^T \boldsymbol{\mathcal{M}} \mathbf{p_i} \leq 0$, i.e., $\text{NPP}(C_j, C_i) \leq 0$.
\end{proof}

\begin{theorem}
For an adjacency matrix $\boldsymbol{\mathcal{A}}$ where $\boldsymbol{\mathcal{A}}_{i,j} = 1$ if there exists a directed path from node $n_i$ to node $n_j$, the number of paths of length $k$ that starts from node $n_i$ and ends on node $n_j$ is given by $(\boldsymbol{\mathcal{A}}^k)_{i,j}$, where length is defined as the number of edges in that path.
\end{theorem}
\begin{proof}
By induction, the base case is when $k = 1$. A path from node $n_i$ to node $n_j$ of length $1$ is simply a direct edge connecting $n_i$ to $n_j$. The presence of such edge is indicated by $\boldsymbol{\mathcal{A}}_{i,j}$ which follows from the definition of adjacency matrix, thereby proving the base case. The inductive step is then to prove that if $(\boldsymbol{\mathcal{A}}^k)_{i,j}$ is the number of paths of length $k$ from node $n_i$ to $n_j$ for some $k \in \mathbb{N}_1$, then $(\boldsymbol{\mathcal{A}}^{k+1})_{i,j}$ is the number of paths of length $k+1$ from node $n_i$ to $n_j$. This follows from the matrix multiplication. Let $\boldsymbol{\mathcal{R}} = \boldsymbol{\mathcal{A}}^{k+1} = \boldsymbol{\mathcal{A}}^k \boldsymbol{\mathcal{A}}$. Then $\boldsymbol{\mathcal{R}}_{i,j} = \sum_{r = 1}^{\vert S_g \vert} \boldsymbol{\mathcal{A}}^k_{i,r} \boldsymbol{\mathcal{A}}_{r,j}$. The inner multiplicative term of the sum will only be non-zero if both $\boldsymbol{\mathcal{A}}^k_{i,r}$ and $\boldsymbol{\mathcal{A}}_{r,j}$ are non-zero, indicating that there exists a path of length $k$ from node $n_i$ to some node $n_r$, and then a direct edge from that node $n_r$ to $n_j$. This forms a path of length $k+1$ from node $n_i$ to $n_j$. Therefore, $\boldsymbol{\mathcal{R}}_{i,j}$ counts the number of paths of length $k+1$ from node $n_i$ to $n_j$, proving the inductive hypothesis.
\end{proof}

\end{document}